\documentclass[final,3p]{elsarticle}
\usepackage{geometry}
\usepackage{mathrsfs,enumerate,amssymb,amsmath,bm,amsthm,color,lineno}
\usepackage[all]{xy}
\usepackage{latexsym}
\usepackage{amscd,verbatim,cuted}
\usepackage{graphicx}
\usepackage[colorlinks=true]{hyperref}
\usepackage{amssymb}
\usepackage{placeins}
\usepackage{float}
\usepackage{varioref}
\usepackage{booktabs}
\usepackage{pifont}
\usepackage{subfigure}

\theoremstyle{plain}
\newtheorem{theorem}{Theorem}[section]

\newtheorem{corollary}{Corollary}[section]

\newtheorem{lemma}{Lemma}[section]
\newtheorem{proposition}{Proposition}[section]
\theoremstyle{definition}
\newtheorem{definition}{Definition}[section]
\newtheorem{example}{Example}[section]

\newtheorem{remark}{Remark}
\newtheorem{property}{Property}

\newtheorem{problem}{Problem}

\numberwithin{equation}{section}

\begin{document}

\begin{frontmatter}
\title{Topological and Algebraic Structures of Atanassov's
Intuitionistic Fuzzy-Values Space\tnoteref{mytitlenote}}
\tnotetext[mytitlenote]{This work was supported by the National Natural Science Foundation of China
(No. 11601449), and the Key Natural Science Foundation of Universities in Guangdong Province (No. 2019KZDXM027).}


\author[a1,a2]{Xinxing Wu}
\address[a1]{School of Sciences, Southwest Petroleum University, Chengdu, Sichuan 610500, China}
\address[a2]{Zhuhai College of Jilin University, Zhuhai, Guangdong 519041, China}
\ead{wuxinxing5201314@163.com}

\author[a1]{Tao Wang}

\author[a1]{Qian Liu}

\author[a3]{Peide Liu}
\address[a3]{School of Management Science and Engineering, Shandong University
of Finance and Economics, Jinan Shandong 250014, China}
\ead{peide.liu@gmail.com}

\author[a4]{Guanrong Chen}
\address[a4]{Department of Electrical Engineering, City University of
Hong Kong, Hong Kong SAR, China}
\ead{eegchen@cityu.edu.hk}

\author[a5]{Xu Zhang\corref{mycorrespondingauthor}}
\cortext[mycorrespondingauthor]{Corresponding author}
\address[a5]{Department of Mathematics, Shandong University, Weihai, Shandong 264209, China}
\ead{xu$\_$zhang$\_$sdu@mail.sdu.edu.cn}
%

\begin{abstract}
We prove that the space of intuitionistic fuzzy values (IFVs) with a
linear order based on a score function and an accuracy function has the same
algebraic structure as the one induced by a linear order based on a similarity
function and an accuracy function. By introducing a new operator for IFVs via
the linear order based on a score function and an accuracy function, we show
that such an operator is a strong negation on IFVs. Moreover, we observe that
the space of IFVs is a complete lattice and a Kleene algebra with the new
operator. We also demonstrate that the topological space of IFVs with the order
topology induced by the above two linear orders is not separable
and metrizable but compact and connected. From some new perspectives,
our results partially answer three open problems posed by Atanassov [Intuitionistic
Fuzzy Sets: Theory and Applications, Springer, 1999] and [On Intuitionistic
Fuzzy Sets Theory, Springer, 2012]. Furthermore, we construct an isomorphism
between the spaces of IFVs and q-rung orthopedic fuzzy values (q-ROFVs) under
the corresponding linear orders. To this end, we introduce the concept of
admissible similarity measures with particular orders for IFSs, extending the
existing definition of the similarity measure for IFSs, and construct an admissible
similarity measure with a linear order based on a score function and an accuracy
function, which is effectively applied to a pattern recognition problem about
the classification of building materials.
\end{abstract}
\end{frontmatter}

\section{Introduction}
Atanassov~\cite{Ata1986,Ata1999} extended Zadeh's fuzzy set
theory by introducing the concept of intuitionistic fuzzy sets (IFSs), characterized
by a membership function and a non-membership function meeting the condition
that the sum of the membership degree and the non-membership degree at every point
is less than or equal to one. Atanassov and Gargov~\cite{AG1989} further extended IFSs
to interval-valued intuitionistic fuzzy sets (IVIFSs) (see also \cite{Ata2020}), with
membership degree and non-membership degree being the closed intervals contained in the
unit interval rather than the real numbers in the unit interval.
Gau and Buehrer~\cite{GB1993} introduced the concept of vague sets, which was proven
to be equivalent to IFSs by Bustince and Burillo~\cite{BB1996}.
Every pair of membership and non-membership degrees for IFSs was
called an intuitionistic fuzzy value (IFV) by Xu~\cite{Xu2007}.
However, in the theory of IFSs, the condition that the sum of the
membership degree and the non-membership degree is less than or
equal to one induces some decision evaluation information that cannot be expressed effectively. Hence, the range of their applications is limited. To overcome this shortcoming,
Yager~\cite{Yager2014,Yager2017,YagerA2017} proposed the concepts of Pythagorean fuzzy sets
(PFSs) and q-rung orthopair fuzzy sets (q-ROFSs). These sets satisfy the condition
that the square sum or the $q$th power sum of the membership degree and the
non-membership degree is less than or equal to one.

Atanassov~\cite{Ata1986} and De et al.~\cite{DBR2000} further introduced some basic operational laws for IFSs,
including ``intersection", ``union", ``supplement", ``sum", ``product", ``power",
``necessity" and ``possibility" operators. Since all operations on IFSs
act on IFVs point by point, it is essential to study the operational laws for IFVs,
which inspired Xu and Yager~\cite{XY2006} to define some of them.
To obtain effective decision-making under an intuitionistic fuzzy environment, it is needed to rank any two IFVs.
For this reason, Xu and Yager~\cite{XY2006} introduced a linear order
`$\leq_{_{\text{XY}}}$' for ranking any two IFVs by using a score function and an accuracy function.
Then, by using a similarity function and an accuracy function, Zhang and Xu~\cite{ZX2012}
introduced another linear order `$\leq_{_{\text{ZX}}}$' to compare any two IFVs. For more results on
linear orders for ranking IFVs or IVIFVs, one is referred  to \cite{Guo2014,DeBPBDaBMO2016,DeBFIKM2016,XXL2017,AFMMF2019}.
Based on linear orders, the above operational laws for IFSs and IVIFSs were
successfully applied to  intuitionistic fuzzy information aggregation
\cite{XC2012,WL2012,DGM2017,XX2011,XXZ2012,WY2016} and decision-making
\cite{GLMG2015,Liu2014,LCW2020,LW2007,XZ2016,Li2010}. In practical applications, the
linear orders $\leq_{_{\text{XY}}}$ and $\leq_{_{\text{ZX}}}$ are fundamental for decision-making
under the intuitionistic fuzzy environment.

With the development of the IFSs theory, more and more unsolved or unformulated problems
were identified. In particular, Atanassov~\cite{Ata1999,Ata2012} posed the following
three open problems:

\begin{problem}[{\textrm{\protect\cite[Open problem~15]{Ata2012},\cite[Open problem~3]{Ata1999}}}]
\label{Prob-1}
What other operations, relations, operators, norms,
and metrics (essential from the standpoint of the applications of IFSs) can be
defined over IFSs and their extensions, and what properties will they
have?
\end{problem}

\begin{problem}[{\textrm{\protect\cite[Open problem~18]{Ata2012}}}]
\label{Prob-2}
To study IF algebraic objects.
\end{problem}

\begin{problem}[{\textrm{\protect\cite[Open problem~22]{Ata2012}}}]
\label{Prob-3}
To introduce elements of IFS in topology and geometry.
\end{problem}

As mentioned above, all operations on IFSs act on IFVs point by point,
so it is essential to study the operational laws on IFVs.
In the following, we address Problems~\ref{Prob-1}--\ref{Prob-3} for IFVs.
Although the linear order $\leq_{_{\text{XY}}}$
was introduced by Xu and Yager~\cite{XY2006} many years ago, the topological and algebraic
structures of the space of IFVs under this linear order have never been investigated in the field of vision.
Therefore, we will systematically study the topological and algebraic
structures of the space of IFVs under some common linear orders discussed in~\cite{Xu2007,XY2006,ZX2012,XC2012}.

The rest of this paper is organized as follows. In Section~\ref{Sec-2}, we introduce
some basic concepts of orders, lattices, and IFSs. In Section~\ref{Sec-3}, we show
that the spaces of IFVs under linear orders $\leq_{_{\text{XY}}}$ and
$\leq_{_{\text{ZX}}}$ are isomorphic; namely, they have the same algebraic
structures. In Section~\ref{Sec-4}, we define a new operator for IFVs by using the
linear order $\leq_{_{\text{XY}}}$, which is a strong negation
on IFVs under this linear order. Furthermore, we prove that
the space of IFVs is a complete lattice and a Kleene algebra under the linear order
$\leq_{_{\text{XY}}}$ with a new strong negation. In Section~\ref{Sec-5},
we demonstrate that the space of IFVs under the order topology with
the linear order $<_{_{\text{XY}}}$ is not separable and metrizable but compact and connected.
These results partially answer Problems~\ref{Prob-1}--\ref{Prob-3}
from some new perspectives. Then, in Section~\ref{Sec-6}, we construct
an isomorphism between the space of IFVs and the space of q-ROFNs under
some linear orders proposed in \cite{LW2018,XZZW2019}, which transforms q-ROFNs to IFVs
equivalently. In Section~\ref{Sec-7}, as applications, we define an admissible
similarity measure with particular orders for IFSs, extending the
similarity measure for IFSs in~\cite{LC2002}. We also construct an admissible
similarity measure with the linear order $\leq_{_{\text{XY}}}$ effectively applied to a pattern
recognition problem about the classification of building materials. Finally, we make some concluding
remarks in Section~\ref{Sec-8}.



\section{Preliminaries}\label{Sec-2}
\subsection{Order}
\begin{definition}[{\textrm{\protect\cite[Definition~1.1.3]{HWW2016}}}]
A {\it partial order} is a binary relation $\preceq$ on a set $L$ with the
following properties:
\begin{enumerate}[(1)]
  \item (Reflexivity) $a\preceq a$.
  \item (Antisymmetry) If $a\preceq b$ and $b\preceq a$, then $a=b$.
  \item (Transitivity) If $a\preceq b$ and $b\preceq c$,  then $a\preceq c$.
\end{enumerate}
\end{definition}

A \textit{partially ordered set}, or \textit{poset} for short, is a nonempty
set $L$ equipped with a partial order $\preceq $.
A {\it bounded poset} is a structure $(L, \preceq, \bm{0}, \bm{1})$ such that
$(L, \preceq)$ is a poset, and $\bm{0}$, $\bm{1} \in L$ are its
bottom and top elements, respectively.

A function $f: X\rightarrow Y$ between two posets is called {\it order preserving} or
{\it monotone} if and only if $x\preceq y$ implies $f(x)\preceq f(y)$.
A bijection $f: X\rightarrow Y$ is called an {\it isomorphism} if $f$ and $f^{-1}$ are monotone.
Two posets $(X, \preceq)$ and $(Y, \preceq)$ are called ({\it order}-)
{\it isomorphic}, denoted by $X\simeq Y$, if and only if there is an isomorphism between them.

Let $(L,\preceq )$ be a poset and $X\subset L$. An element $u\in L$ is said
to be an \textit{upper bound} of $X$ if $x\preceq u$ for all $x\in X$. An
upper bound $u$ of $X$ is said to be its \textit{smallest upper bound} or
\textit{supremum}, written as $\bigvee X$ or $\sup X$, if $u\preceq y$ for
all upper bounds $y$ of $X$. Similarly, we can define the \textit{greatest
lower bound} or \textit{infimum} of $X$, written as $\bigwedge X$ or $\inf X$.
For any pair of elements, simply write
\begin{equation}
x\vee y=\sup \{x,y\}\text{ and }x\wedge y=\inf \{x,y\}.
\label{sup-inf-operation}
\end{equation}

For $x, y \in X$, the notation $x\prec y$ means that $x\preceq y$ and $x\neq y$.
Suppose that $(X, \preceq)$ is a poset. Given elements $a, b \in X$ satisfying
$a\preceq b$, there exist some subsets of $X$, called the {\it intervals},
as below:
\begin{align*}
(a, b)&=\left\{x\in X \mid a\prec x\prec b\right\},\\
(a, b]&=\left\{x\in X \mid a\prec x\preceq b\right\},\\
[a, b)&=\left\{x\in X \mid a\preceq x\prec b\right\},\\
[a, b]&=\left\{x\in X \mid a\preceq x\preceq b\right\},\\
(\leftarrow, a]&=\left\{x\in X \mid x\preceq a\right\},\\
(\leftarrow, a)&=\left\{x\in X \mid x\prec a\right\},\\
[b, \rightarrow)&=\left\{x\in X \mid b\preceq x\right\},\\
(b, \rightarrow)&=\left\{x\in X \mid b\prec x\right\}.
\end{align*}

\begin{definition}
[{\textrm{\protect\cite{Eng1977, M1975}}}]
A \textit{linear order} is a binary relation $\prec$ on a set $X$ with the following properties:
\begin{itemize}
\item[(LO1)]  (Comparability) If $x \neq y$, then either $x \prec y$ or $y \prec x$.
\item[(LO2)]  (Nonreflexivity) For no $x$ in $X$, the relation $x \prec x$ holds.
\item[(LO3)]  (Transitivity) If  $x \prec y$ and  $y \prec z$, then  $x \prec z$.
\end{itemize}
A \textit{linearly ordered set} is a nonempty set $X$ together with a linear order $\prec$.
\end{definition}

\begin{definition}[{\textrm{\protect\cite{M1975}}}]
\label{Def-order-Topology}
Let $X$ be a linearly ordered set with more than one element. Assume that
$\mathscr{B}$ is the collection of all sets of the following types:
\begin{enumerate}[(1)]
\item  All open intervals $(a, b)$ in $X$.
\item  All intervals of the form $(\leftarrow, b)$.
\item  All intervals of the form $(a, \rightarrow)$.
\end{enumerate}
The collection $\mathscr{B}$ is a basis for a topology on $X$,
called the \textit{linear order topology}.
\end{definition}

Any pair $(D, E)$ of subsets of $X$ that is a linearly ordered set together with
$\prec$ is called a \textit{cut} of $X$ if the following properties hold:
(1) $D\cup E=X$;
(2) $D\neq \varnothing$ and $E\neq \varnothing$;
(3) $x\prec y$ for all $x\in D$ and $y\in E$.
The sets $D$ and $E$ are called the \textit{lower section} and the \textit{upper section} of the cut, respectively.
These sections are disjoint. For every cut of a linearly ordered set,
one of the following four conditions is satisfied:
\begin{enumerate}[(i)]
\item\label{Cut-i} The lower and the upper sections have the largest and the smallest elements, respectively.
\item The lower section has the largest element while the upper section has no smallest one.
\item The upper section has the smallest element while the lower section has no largest one.
\item\label{Cut-iv} The lower and the upper sections have no largest and smallest elements, respectively.
\end{enumerate}
When condition (\ref{Cut-i}) (resp. (\ref{Cut-iv})) holds, we say that the cut is a \textit{jump} (resp. \textit{gap}).
If no cut of a linearly ordered set $X$ is a jump (resp. gap), then it is called a \textit{densely ordered set} (resp.
\textit{continuously ordered set}).

\subsection{Lattice}
\begin{definition}[{\textrm{\protect\cite{Bir1967}}}]\label{de-lattice}
A \textit{lattice} is a poset such that every pair of elements
have the greatest lower bound and the smallest upper bound. A lattice
is \textit{bounded} if it has the bottom and the top elements.
\end{definition}

\begin{definition}[{\textrm{\protect\cite{Bir1967}}}]\label{de-algebra}
A \textit{lattice} is an algebra $(L, \vee, \wedge)$ with two binary
operations $\wedge$ and $\vee$, which are called \textit{meet} and \textit{join}, respectively,
if the following conditions hold:
\begin{itemize}
\item[(L1)] (Idempotent law) $x\wedge x=x$ and $x\vee x=x$.
\item[(L2)] (Commutative law) $x\wedge y=y\wedge x$ and $x\vee y=y\vee x$.
\item[(L3)] (Associative law) $x\wedge(y\wedge z)=(x\wedge y)\wedge z$ and
$x\vee(y\vee z)=(x\vee y)\vee z$;
\item[(L4)] (Absorption law) $x\wedge(x\vee y)=x$ and $x\vee(x\wedge y)=x$.
\end{itemize}
\end{definition}

For the equivalence of Definitions \ref{de-lattice} and \ref{de-algebra},
we refer to {\textrm{\protect\cite[Theorem~1]{Bir1967}}}. In fact,
for a lattice, the meet $\wedge$ and join $\vee$ operations in
Definition~\ref{de-algebra} can be defined as the infimum and supremum of two elements, respectively.

\begin{definition}
\label{Complete-Lattice}
A lattice $(L, \preceq)$ is \textit{complete} if every subset $A$ of $L$ has
a supremum and an infimum under the partial order $\preceq$.
\end{definition}

\begin{lemma}[{\textrm{\protect\cite{Bir1967}}}]
\label{Comp-Lattice-Char}
For a bounded lattice $L$, the following statements are equivalent:
\begin{enumerate}[{\rm (i)}]
\item $L$ is a complete lattice;
\item Every nonempty subset of $L$ has an infimum.
\item Every nonempty subset of $L$ has a supremum.
\end{enumerate}
\end{lemma}

\begin{definition}[{\textrm{\protect\cite[Definition~1.1.4]{HWW2016}}}]
A \textit{distributive lattice} is a lattice $(L, \vee, \wedge)$
satisfying the following distributive laws: for any $x$, $y$, $z\in L$,
\begin{enumerate}[(1)]
\item $x\wedge(y\vee z)=(x\wedge y)\vee(x\wedge z)$;
\item $x\vee(y\wedge z)=(x\vee y)\wedge(x\vee z)$.
\end{enumerate}
\end{definition}

\begin{definition}[{\textrm{\protect\cite[Definition~1.1.6]{HWW2016}}}]
Let $L$ be a bounded lattice. A \textit{negation} on $L$ is a decreasing
mapping $N: L\rightarrow L$ such that $N(\bm{0})=\bm{1}$
and $N(\bm{1})=\bm{0}$. If, additionally, $N(N(x))=x$
holds for all $x\in L$, then it is called a \textit{strong negation}.
\end{definition}

\begin{definition}[{\textrm{\protect\cite[Definition~1.1.7]{HWW2016}}}]
A \textit{De Morgan algebra} is a bounded distributive lattice
with a strong negation $\neg$.
\end{definition}

\begin{definition}[{\textrm{\protect\cite[Definition~1.1.8]{HWW2016}}}]
A \textit{Kleene algebra} is a De Morgan algebra $(L, \vee, \wedge, \neg)$ that satisfies
Kleene's inequality: for any $x, y\in L$,
$$
x\wedge (\neg x)\preceq y \vee (\neg y).
$$
\end{definition}

\subsection{Intuitionistic fuzzy set (IFS)}

\begin{definition}[{\textrm{\protect\cite[Definition~1.1]{Ata1999}}}]
Let $X$ be the universe of discourse. An \textit{intuitionistic fuzzy set}~(IFS)
$I$ in $X$ is defined as an object in the following form:
\begin{equation}
I=\left\{\langle x, \mu_{_{I}}(x), \nu_{_{I}}(x)\rangle \mid x\in X\right\},
\end{equation}
where the functions
$$
\mu_{_{I}}: X \rightarrow [0,1],
$$
and
$$
\nu_{_{I}}: X \rightarrow [0,1],
$$
define the \textit{degree of membership} and the
\textit{degree of non-membership} of the element $x \in X$ to the set $I$,
respectively, and for every $x\in X$,
\begin{equation}
\mu_{_{I}}(x)+\nu_{_{I}}(x)\leq 1.
\end{equation}
\end{definition}

Let $\mathrm{IFS}(X)$ denote the set of all IFSs in the universe of discourse $X$.
For $I\in \mathrm{IFS}(X)$, the \textit{indeterminacy degree} $\pi_{_{I}}
(x)$ of an element $x$ belonging to $I$ is defined by $\pi_{_I}(x)=1-\mu_{_I}(x)-\nu_{_I}(x)$.
In \cite{Xu2007,XC2012}, the pair $\langle\mu_{_I}(x), \nu_{_I}(x)\rangle$ is called an \textit{ intuitionistic fuzzy value} (IFV) or an \textit{intuitionistic fuzzy number} (IFN).
For convenience, we use $\alpha=\langle \mu_{\alpha}, \nu_{\alpha}\rangle$ to represent an IFV $\alpha$,
which satisfies $\mu_{\alpha}\in [0, 1]$, $\nu_{\alpha}\in [0, 1]$, and $0\leq \mu_{\alpha}
+\nu_{\alpha}\leq 1$. Additionally, $s(\alpha) =\mu_{\alpha}-\nu_{\alpha}$ and $h(\alpha)=\mu_{\alpha}+\nu_{\alpha}$
are called the \textit{score degree} and the \textit{accuracy degree} of $\alpha$, respectively.
Let $\tilde{\mathbb{I}}$ denote the set of all IFVs, i.e.,
${\tilde{\mathbb{I}}}=\{\langle \mu, \nu \rangle\in [0, 1]^{2} \mid \mu+\nu \leq 1\}$.

Motivated by the basic operations on IFSs, Xu et al.~\cite{XY2006,XC2012}
introduced the following basic operational laws for IFVs.

\begin{definition}[{\textrm{\protect\cite[Definition~1.2.2]{XC2012}}}]
\label{Def-Int-Operations}
Let $\alpha=\langle\mu_{\alpha}, \nu_{\alpha}\rangle$,
$\beta=\langle\mu_{\beta}, \nu_{\beta}\rangle\in \tilde{\mathbb{I}}$. Define
\begin{enumerate}[(i)]
\item $\overline{\alpha}=\langle\nu_{\alpha},\mu_{\alpha}\rangle$.
\item $\alpha\cap\beta=\langle\min\{\mu_{\alpha}, \mu_{\beta}\}, \max\{\nu_{\alpha},\nu_{\beta}\}\rangle$.
\item $\alpha\cup\beta=\langle\max\{\mu_{\alpha},\mu_{\beta}\}, \min\{\nu_{\alpha},\nu_{\beta}\}\rangle$.
\item $\alpha\oplus\beta=\langle\mu_{\alpha}+\mu_{\beta}-\mu_{\alpha}\mu_{\beta}, \nu_{\alpha}\nu_{\beta}\rangle$.
\item $\alpha\otimes\beta=\langle\mu_{\alpha}\mu_{\beta}, \nu_{\alpha}+\nu_{\beta}-\nu_{\alpha}\nu_{\beta}\rangle$.
\item $\lambda\alpha=\langle 1-(1-\mu_{\alpha})^{\lambda}, (\nu_{\alpha})^{\lambda}\rangle$, $\lambda >0 $.
\item $\alpha^{\lambda}=\langle (\mu_{\alpha})^{\lambda}, 1-(1-\nu_{\alpha})^{\lambda}\rangle$, $\lambda >0$.
\end{enumerate}
\end{definition}

The order $\subset$, defined by the rule that $\alpha\subset \beta$ if and only if
$\alpha\cap \beta=\alpha$, is a partial order on $\tilde{\mathbb{I}}$.
To compare any two IFVs, Xu and Yager~\cite{XY2006} introduced the following linear order `$\leq_{_{\text{XY}}}$'
(see also \cite[Definition~3.1]{Xu2007} and \cite[Definition~1.1.3]{XC2012}):

\begin{definition}
[{\textrm{\protect\cite[Definition~1]{XY2006}}}]
\label{de-order(Xu)}
Let $\alpha_{1}$ and $\alpha_{2}$ be two IFVs.
\begin{itemize}
  \item If $s(\alpha_{1})<s(\alpha_{2})$,
then $\alpha_{1}$ is smaller than $\alpha_{2}$,
denoted by $\alpha_{1}<_{_{\text{XY}}}\alpha_{2}$.

  \item If $s(\alpha_{1})=s(\alpha_{2})$, then

\begin{itemize}
  \item[$-$] if $h(\alpha_{1})=h(\alpha_{2})$, then $\alpha_{1}=\alpha_{2}$;
  \item[$-$] if $h(\alpha_{1})<h(\alpha_{2})$, then $\alpha_{1}<_{_{\text{XY}}}\alpha_{2}$.
\end{itemize}
\end{itemize}
If $\alpha_{1}<_{_{\text{XY}}}\alpha_{2}$ or $\alpha_{1}=\alpha_{2}$,
 then denote it as $\alpha_{1}\leq_{_{\text{XY}}}\alpha_{2}$.
\end{definition}

Alongside Xu and Yager's order `$\leq_{_{\text{XY}}}$' in Definition \ref{de-order(Xu)},
Szmidt and Kacprzyk~\cite{SK2009} proposed another comparison
function $\rho(\alpha)=\frac{1}{2}(1+\pi(\alpha))(1-\mu(\alpha))$ for IFVs,
which is a partial order. However, it sometimes cannot distinguish between any two IFVs.
Although Xu's method \cite{Xu2007} constructs a linear order for ranking any pair of IFVs, this procedure has
the following disadvantages:
(1) It may result in that the less we know, the better the IFV, which contradicts the general logic.
(2) It is sensitive to a slight change of parameters.
(3) It is not preserved under multiplication by a scalar, namely,
$\alpha\leq_{_{\text{XY}}}\beta$ might not imply
$\lambda\alpha<_{_{\text{XY}}}\lambda\beta$, where $\lambda$ is a scalar
(see {\textrm{\protect\cite[Example~1]{BBGMP2011}}}).
To overcome some shortcomings of the above two ranking methods, Zhang and Xu~\cite{ZX2012}
improved Szmidt and Kacprzyk's one \cite{SK2009}, according to Hwang and Yoon's idea
\cite{HY1981} on the technique for order preference by similarity to an ideal solution.
They also defined a similarity function $L(\alpha)$, for any IFV
$\alpha=\langle\mu_{\alpha}, \nu_{\alpha}\rangle$, as follows:
\begin{equation}
L(\alpha)=\frac{1-\nu_{\alpha}}{(1-\mu_{\alpha})+(1-\nu_{\alpha})}
=\frac{1-\nu_{\alpha}}{1+\pi_{\alpha}}.
\end{equation}
In particular, if $\nu_{\alpha}<1$, then
\begin{equation}\label{L-formula}
L(\alpha)=\frac{1}{\tfrac{1-\mu_{\alpha}}{1-\nu_{\alpha}}+1}.
\end{equation}
Furthermore, Zhang and Xu~\cite{ZX2012} introduced the following order `$\leq_{_{\textrm{ZX}}}$'
for IFVs by applying the similarity function $L(\_)$.

\begin{definition}[{\textrm{\protect\cite{ZX2012}}}]
\label{de-order(ZX)}
Let $\alpha_{1}$ and $\alpha_{2}$ be two IFVs.
\begin{itemize}
  \item If $L(\alpha_{1})<L(\alpha_{2})$, then $\alpha_{1}$ is smaller than $\alpha_{2}$,
denoted by $\alpha_{1}<_{_{\textrm{ZX}}}\alpha_{2}$;
  \item If $L(\alpha_{1})=L(\alpha_{2})$, then
  \begin{itemize}
    \item[$-$] if $h(\alpha_{1})=h(\alpha_{2})$, then $\alpha_{1}=\alpha_{2}$;
    \item[$-$] if $h(\alpha_{1})<h(\alpha_{2})$, then $\alpha_{1}<_{_{\textrm{ZX}}}\alpha_{2}$.
  \end{itemize}
\end{itemize}
If $\alpha_{1}<_{_{\textrm{ZX}}}\alpha_{2}$ or $\alpha_{1}=\alpha_{2}$, then denote it as $\alpha_{1}\leq_{_{\textrm{ZX}}}\alpha_{2}$.
\end{definition}

\begin{remark}\label{Remark-1}
The bottom and the top elements of $\tilde{\mathbb{I}}$ are $\langle 0, 1\rangle$
and $\langle 1, 0 \rangle$, respectively, under both the order $\leq_{_{\text{XY}}}$ and the order $\leq_{_{\textrm{ZX}}}$.
\end{remark}


\section{An isomorphism between lattices $(\tilde{\mathbb{I}}, \leq_{_{\text{XY}}})$
and $(\tilde{\mathbb{I}}, \leq_{_{\textmd{ZX}}})$}\label{Sec-3}

In this section, we construct an isomorphism between $(\tilde{\mathbb{I}}, \leq _{_\text{XY}})$ and
$(\tilde{\mathbb{I}}, \leq _{_\mathrm{ZX}})$. Therefore, we only consider the construction of
$(\tilde{\mathbb{I}}, \leq _{_\text{XY}})$ in the following sections.

Define a mapping $\tilde{\Gamma}: (\tilde{\mathbb{I}},\leq _{_\mathrm{ZX}})
\rightarrow (\tilde{\mathbb{I}}, \leq _{_\text{XY}}) $ as follows:
for $\alpha=(\mu_\alpha, \nu_\alpha) \in \tilde{\mathbb{I}}$,
\begin{equation}
\label{eq-3.1}
\tilde{\Gamma}(\alpha)
=
\begin{cases}
\langle 0, 1\rangle, & L(\alpha)=0, \\
\langle1, 0\rangle, & L(\alpha)=1, \\
\alpha, & L(\alpha)=\frac{1}{2}, \\
\left\langle \frac{\mu_{\alpha}}{2(1-L(\alpha))} ,
\frac{\mu_{\alpha} + 2-4L(\alpha)}{2(1-L(\alpha))}\right\rangle,
& 0< L(\alpha) <\frac{1}{2}, \\
\Big\langle \frac{1}{2}\Big(\frac{\mu_{\alpha}}{1-L(\alpha)}
-\frac{(2L(\alpha)-1)^2}{L(\alpha)(1-L(\alpha))}\Big),& \\
~~~~~~~~\frac{L(\alpha)\mu_{\alpha}-(2L(\alpha)-1)}{2L(\alpha)(1-L(\alpha))}\Big\rangle,
& \frac{1}{2} < L(\alpha) <1.
\end{cases}
\end{equation}
Figure~\ref{Fig-1} is the geometric description of $\tilde{\Gamma}(\alpha)$.
\begin{figure}[H]
\begin{center}
\scalebox{0.55}{\includegraphics{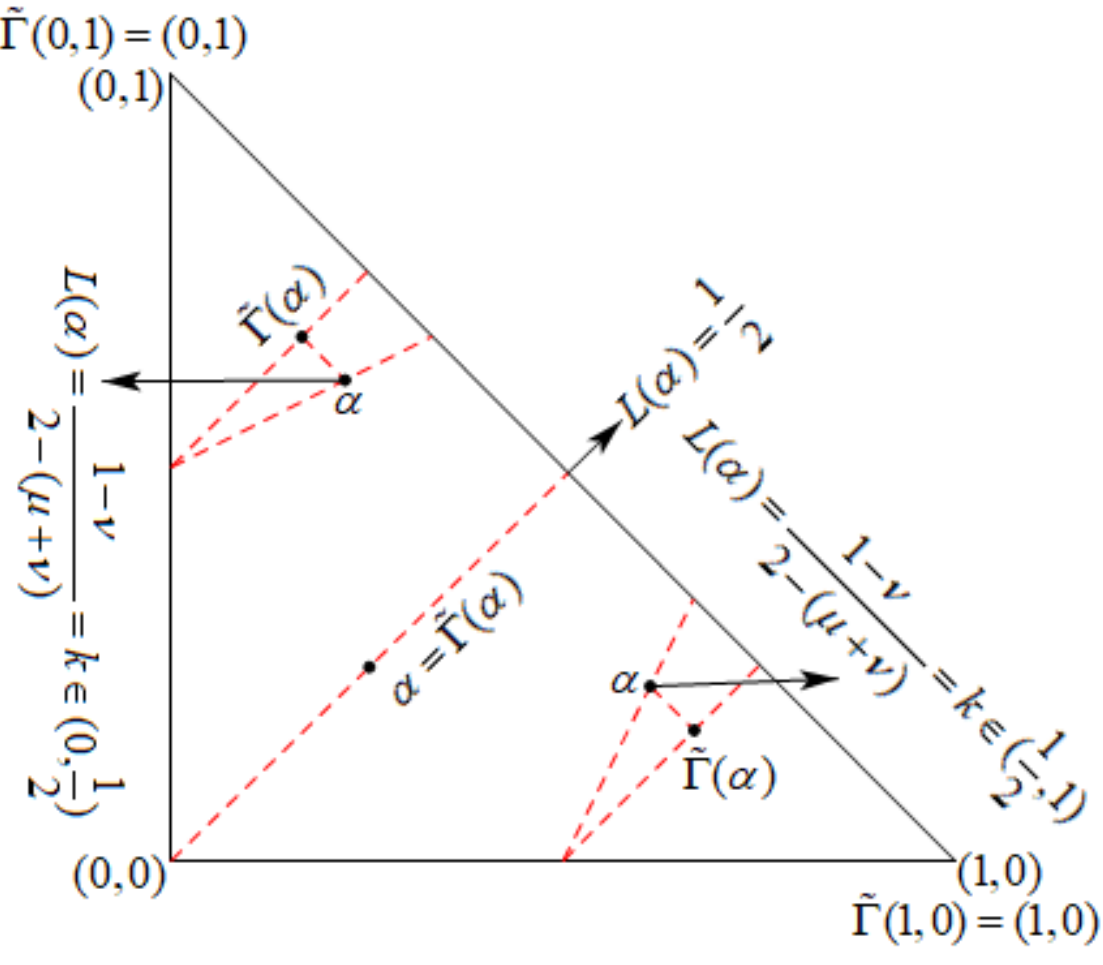}}
\end{center}
\caption{Geometrical interpretation of $\tilde{\Gamma}(\alpha)$}
\label{Fig-1}
\end{figure}

\begin{theorem}\label{th-isomorphism}
The mapping $\tilde{\Gamma}$ defined in~\eqref{eq-3.1} is an
isomorphism between $(\tilde{\mathbb{I}}, \leq _{_\mathrm{ZX}})$
and $(\tilde{\mathbb{I}}, \leq_{_\text{XY}})$.
\end{theorem}

\begin{proof}
By formula~\eqref{eq-3.1}, $\tilde{\Gamma}$ is a bijection.
For $\alpha=\langle \mu_\alpha, \nu_\alpha\rangle$ and $\beta=\langle \mu_\beta, \nu_\beta\rangle
\in \tilde{\mathbb{I}}$ such that $\alpha<_{_\mathrm{ZX}} \beta$, we prove that
$\tilde{\Gamma}(\alpha)<_{_\text{XY}} \tilde{\Gamma}(\beta)$. For this purpose,
we consider the following cases:

(1) If $L(\alpha)=0$ or $L(\beta)=1$, i.e., $\alpha=\langle 0, 1\rangle$ or $\beta=\langle 1, 0 \rangle$,
then $\tilde{\Gamma}(\alpha)=\langle 0, 1 \rangle$ or $\tilde{\Gamma}(\beta)=\langle 1, 0\rangle$. Thus,
$\tilde{\Gamma}(\alpha)<_{_\text{XY}} \tilde{\Gamma}(\beta)$.

(2) If $L(\alpha)= L(\beta)=\frac{1}{2}$, then
$\tilde{\Gamma}(\alpha)=\alpha$ and $\tilde{\Gamma}(\beta)=\beta$ by formula~\eqref{eq-3.1}.
Since $\alpha <_{_\mathrm{ZX}} \beta$, we obtain $h(\alpha)<h(\beta)$, which
implies that $\alpha <_{_\text{XY}} \beta$. Thus,
$\tilde{\Gamma}(\alpha)<_{_\text{XY}} \tilde{\Gamma}(\beta)$.

(3) If $0 < L(\alpha)= L(\beta) < \frac{1}{2}$, then, by formula \eqref{eq-3.1}, we have
    $$
    \tilde{\Gamma}(\alpha)=\left\langle \frac{\mu_{\alpha}}{2(1-L(\alpha))},
    \frac{\mu_{\alpha} + 2-4L(\alpha)}{2(1-L(\alpha))} \right\rangle,
    $$
and
    $$
    \tilde{\Gamma}(\beta)=\left\langle \frac{\mu_{\beta}}{2(1-L(\beta))},
    \frac{\mu_{\beta} + 2-4L(\beta)}{2(1-L(\beta))} \right\rangle.
    $$
Since $\alpha<_{_\mathrm{ZX}} \beta$, we have $h(\alpha) < h(\beta)$. This, together with
$\frac{1-\nu_{\alpha}}{2-h(\alpha)}=L(\alpha)=L(\beta)=\frac{1-\nu_{\beta}}{2-h(\beta)}$, implies
that $\nu_{\alpha}<\nu_{\beta}$. Thus, $\mu_{\alpha}< \mu_{\beta}$ by formula~\eqref{L-formula}. In this case, $h(\tilde{\Gamma}(\alpha))
=\frac{\mu_{\alpha}+1-2L(\alpha)}{1-L(\alpha)}<\frac{\mu_{\beta}+1-2L(\beta)}{1-L(\beta)}=h(\tilde{\Gamma}(\beta))$.
By $s(\tilde{\Gamma}(\alpha))=s(\tilde{\Gamma}(\beta))$, there holds $\tilde{\Gamma}(\alpha) <_{_\text{XY}} \tilde{\Gamma}(\beta)$.

(4) If $\frac{1}{2}<L(\alpha)= L(\beta)<1$, then, similarly to the proof of (3),
it can be verified that $\tilde{\Gamma}(\alpha)<_{_\text{XY}} \tilde{\Gamma}(\beta)$.

(5) If $0<L(\alpha)<L(\beta)\leq\frac{1}{2}$, then, by formula~\eqref{eq-3.1},
we have $s(\tilde{\Gamma}(\alpha))=\frac{2L(\alpha)-1}{1-L(\alpha)} < \frac{2L(\beta)-1}{1-L(\beta)}
=s(\tilde{\Gamma}(\beta))$. Thus, $\tilde{\Gamma}(\alpha)<_{_\text{XY}} \tilde{\Gamma}(\beta)$.

(6) If $\frac{1}{2}\leq L(\alpha)<L(\beta)<1$, then, by formula~\eqref{eq-3.1}, we have
$s(\tilde{\Gamma}(\alpha))=\frac{2L(\alpha)-1}{L(\alpha)}<\frac{2L(\beta)-1}{L(\beta)}
=s(\tilde{\Gamma}(\beta))$. Thus, $\tilde{\Gamma}(\alpha)<_{_\text{XY}} \tilde{\Gamma}(\beta)$.

(7) If $0<L(\alpha)<\frac{1}{2}<L(\beta)<1$, then, by formula~\eqref{eq-3.1}, we have
$s(\tilde{\Gamma}(\alpha))=\frac{2L(\alpha)-1}{1-L(\alpha)}<0<\frac{2L(\beta)-1}{1-L(\beta)}
=s(\tilde{\Gamma}(\beta))$. Thus, $\tilde{\Gamma}(\alpha)<_{_\text{XY}}\tilde{\Gamma}(\beta)$.

By direct calculation, it can be verified that,
for $\alpha=\langle \mu_\alpha, \nu_\alpha\rangle \in \tilde{\mathbb{I}}$,
\begin{equation}
\label{eq-3.2}
\tilde{\Gamma}^{-1}(\alpha)
=
\begin{cases}
\langle 0, 1\rangle, & s(\alpha)=-1, \\
\langle 1, 0\rangle, & s(\alpha)=1, \\
\alpha, & s(\alpha)=0, \\
\Big\langle 2\mu_{\alpha}-s(\alpha)-\frac{2\mu_{\alpha}-2s(\alpha)}{2-s(\alpha)}, & \\
~~~~~~~~~~~~ \frac{2\mu_{\alpha}-2s(\alpha)}{2-s(\alpha)}\Big\rangle, & 0<s(\alpha)<1, \\
\left\langle \frac{2\mu_{\alpha}}{2+s(\alpha)},
\frac{2\mu_{\alpha}(1+s(\alpha))}{2+s(\alpha)}-s(\alpha)\right\rangle, & -1<s(\alpha)<0.
\end{cases}
\end{equation}

Similarly to the above proof, one can show that, for any $\alpha$, $\beta \in\tilde{\mathbb{I}}$
with $\alpha <_{_\text{XY}} \beta$, there holds $\tilde{\Gamma}^{-1}(\alpha)<_{_\mathrm{ZX}}
\tilde{\Gamma}^{-1}(\beta)$. Hence, $\tilde{\Gamma}$ an isomorphism.
\end{proof}

\section{Algebraic structure of $(\tilde{\mathbb{I}}, \leq_{_{\text{XY}}})$}\label{Sec-4}

In  Theorem~\ref{th-isomorphism}, we observe that $(\tilde{\mathbb{I}}, \leq_{_{\text{XY}}})$$\simeq$$(\tilde{\mathbb{I}}, \leq_{_{\text{ZX}}})$; namely, ${(\tilde{\mathbb{I}}, \leq_{_{\text{XY}}})}$ and
$(\tilde{\mathbb{I}}, \leq_{_{\text{ZX}}})$ have the same algebraic structures. Accordingly, it suffices
to study the algebraic structures of $(\tilde{\mathbb{I}}, \leq_{_{\text{XY}}})$.
In the following, we first show that $(\tilde{\mathbb{I}}, \leq_{_{\text{XY}}})$ is a
complete lattice. Then, by introducing the operator `$\neg$' for IFVs, we demonstrate in Corollary~\ref{co-strong-negation} that
it is a strong negation on $(\tilde{\mathbb{I}}, \leq_{_{\text{XY}}})$. We also present in Theorem~\ref{Kleene-algebra-Thm} that
$(\tilde{\mathbb{I}}, \wedge, \vee, \neg)$ is a Kleene algebra.

\begin{proposition}
\label{Pro-order-Xu}
{\rm (i)} The order `$<_{_{\text{XY}}}$' defined in Definition~\ref{de-order(Xu)}
is a linear order on $\tilde{\mathbb{I}}$.

{\rm (ii)} Let $\alpha\leq_{_{\text{XY}}}\beta$ and
$s(\alpha)=s(\beta)$ for $\alpha=\langle \mu_{\alpha}, \nu_{\alpha}\rangle$ and $\beta=\langle\mu_{\beta},
\nu_{\beta}\rangle\in \tilde{\mathbb{I}}$. Then, there holds $(\alpha, \beta)=\{\gamma \in \tilde{\mathbb{I}} \mid
\alpha<_{_{\text{XY}}}\gamma <_{_{\text{XY}}}\beta\}
 =\{(\mu,\nu)\in \tilde{\mathbb{I}} \mid \mu-\nu=s({\alpha}) \text{ and }
 \mu_{\alpha}<\mu<\mu_{\beta}\}$.
\end{proposition}

\begin{theorem}
\label{Complete-Thm-Wu}
$(\tilde{\mathbb{I}}$, $\leq_{_{\text{XY}}})$ is a complete lattice.
\end{theorem}
\begin{proof}
Given a nonempty subset $\Omega \subset \tilde{\mathbb{I}}$, it will be shown that the
greatest lower bound of $\Omega$ exists.

Let $\mathscr{S}(\Omega)=\{\mu_{\alpha}-\nu_{\alpha} \mid \langle\mu_{\alpha}, \nu_{\alpha}\rangle \in \Omega\}$
and $\xi=\inf\mathscr{S}(\Omega)$. Consider the following two cases:

(1) Assume that $\xi\in \mathscr{S}(\Omega)$. Then, there exists
$\alpha_1=\langle \mu_{\alpha_1}, \nu_{\alpha_1}\rangle \in \Omega$ such that
$\mu_{\alpha_1}-\nu_{\alpha_1}=\xi$. This means that the set
$\Omega^{\xi} =\{\mu_{\alpha} \mid \langle\mu_{\alpha}, \mu_{\alpha}-\xi\rangle\in \Omega \}$
is a nonempty set. Let $\hat{\mu}=\inf \Omega^\xi$ and $\hat{\alpha}=\langle\hat{\mu}, \hat{\mu}-\xi\rangle$.
It is obvious that $s(\hat{\alpha})=\xi$. Notice that
\begin{align*}
\hat{\mu}-\xi
&=\inf\Omega^\xi-\xi=\inf\{\mu_{\alpha}-\xi \mid \langle\mu_{\alpha}, \mu_{\alpha}-\xi\rangle\in \Omega\}\\
&=\inf\{\nu_{\alpha} \mid \langle\nu_{\alpha}+\xi, \nu_{\alpha}\rangle\in \Omega\}\geq 0.
\end{align*}
Hence, $\hat{\alpha}\in\tilde{\mathbb{I}}$. Now, we show that $\hat{\alpha}$ is the greatest lower bound of $\Omega$.

1.1) For any $\alpha=\langle \mu_{\alpha},\nu_{\alpha}\rangle\in \tilde{\mathbb{I}}$,
it follows from the choice of $\xi$ that $s(\alpha)\geq \xi=s(\hat{\alpha})$.
(1) If $s(\alpha)> \xi=s(\hat{\alpha})$, then $\alpha>_{_{\text{XY}}}
\hat{\alpha}$; (2) If $s(\alpha)=\xi=s(\hat{\alpha})$, then $\mu_{\alpha}\in\Omega^{\xi}$.
This, together with $\hat{\mu}=\inf\Omega^{\xi}$, implies that $\mu_{\alpha}\geq \hat{\mu}$.
Thus,
 $$
 h(\alpha)=\mu_{\alpha}+\nu_{\alpha}\geq \hat{\mu}=h(\hat{\alpha}).
 $$
Therefore, $\alpha\geq_{_{\text{XY}}}\hat{\alpha}$, i.e., $\hat{\alpha}$ is
a lower bound of $\Omega$.

1.2)  Given a lower bound $\alpha^{\prime}=\langle \mu_{\alpha^\prime}, \nu_{\alpha^\prime}\rangle
\in \tilde{\mathbb{I}}$ of $\Omega$, it follows from Definition~\ref{de-order(Xu)} that
$s(\alpha^\prime)\leq \xi$.

1.2.1)  If $s(\alpha^\prime)<\xi$, then $\alpha^\prime<_{_{\text{XY}}}\hat{\alpha}$
since $s(\hat{\alpha})=\xi$.

1.2.2)  If $s(\alpha^\prime)=\xi $, it follows from Definition~\ref{de-order(Xu)} that,
for any $\alpha=\langle \mu_{\alpha}, \nu_{\alpha}\rangle\in \Omega$ with
$\mu_{\alpha}-\nu_{\alpha}=\xi$, $h(\alpha^\prime) \leq h(\alpha)=\mu_{\alpha}+\nu_{\alpha}$.
It follows that
\begin{align*}
h(\alpha^\prime)
&\leq\inf\{\mu_{\alpha}+\nu_{\alpha} \mid \langle\mu_{\alpha}, \nu_{\alpha}\rangle\in \Omega
\text{ and }  \mu_{\alpha}-\nu_{\alpha}=\xi\} \\
&=\inf\{2\mu_{\alpha}-\xi \mid \langle \mu_{\alpha},\nu_{\alpha}\rangle\in \Omega \text{ and }
\mu_{\alpha}-\nu_{\alpha}=\xi \} \\
&=2\inf \Omega^{\xi}-\xi=2\hat{\mu}-\xi=h(\hat{\alpha}).
\end{align*}
This, together with $s(\alpha^{\prime})=s(\hat{\alpha})$, implies that
$\alpha^{\prime}\leq_{_{\text{XY}}} \hat{\alpha}$.

Summing up 1.1) and 1.2) shows that $\hat{\alpha}$ is the greatest lower bound of $\Omega$ .

(2) Assume that $\xi\notin \mathscr{S}(\Omega)$. Take $\tilde{\alpha}=\big\langle\frac{1+\xi}{2},
\frac{1-\xi}{2}\big\rangle \in \tilde{\mathbb{I}}$. It is obvious that $s(\tilde{\alpha})=\xi$.
Now, we show that $\tilde{\alpha}$ is the
greatest lower bound of $\Omega$.

2.1) It follows from $\xi\notin \mathscr{S}(\Omega)$ that, for any $\alpha\in\Omega$,
 $s(\alpha)>\xi=s(\tilde{\alpha})$. This observation shows that $\alpha>_{_{\text{XY}
 }}\tilde{\alpha}$. Namely, $\tilde{\alpha}$ is a lower bound of $\Omega$.

2.2) Given a lower bound $\alpha^{\prime\prime}=\langle \mu_{\alpha^{\prime\prime}},
\nu_{\alpha^{\prime\prime}}\rangle \in \tilde{\mathbb{I}}$,
it follows from Definition~\ref{de-order(Xu)}
that $s(\alpha^{\prime\prime})\leq \xi$.

2.2.1) If $s(\alpha^{\prime\prime})<\xi$, then
$\alpha^{\prime\prime}<_{_{\text{XY}}}\tilde{\alpha}$ since $s(\tilde{\alpha})=\xi$.

2.2.2) If $s(\alpha^{\prime\prime})=\xi$, then $\alpha^{\prime\prime}\leq_{_{\text{XY}}}\tilde{\alpha}$
since $h(\tilde{\alpha})=\frac{1+\xi}{2}+\frac{1-\xi}{2}=1\geq \mu_{\alpha^{\prime\prime}}
+\nu_{\alpha^{\prime\prime}}=h(\alpha^{\prime\prime})$
and $s(\tilde{\alpha})=\xi$.

Summing up 2.1) and 2.2) shows that $\tilde{\alpha}$ is the greatest lower bound of $\Omega$.

Hence, $(\tilde{\mathbb{I}}$, $\leq_{_{\text{XY}}})$ is a complete lattice by
applying Lemma~\ref{Comp-Lattice-Char}.
\end{proof}

Given any nonempty subset $\Omega\subset\tilde{\mathbb{I}}$, let
$\mathscr{S}(\Omega)=\{\mu_{\alpha}-\nu_{\alpha} \mid \langle\mu_{\alpha},
\nu_{\alpha}\rangle \in\Omega\}$, $\xi(\Omega)=\inf\mathscr{S}(\Omega)$,
and $\eta(\Omega)=\sup\mathscr{S}(\Omega)$.

\begin{remark}
According to the proof of Theorem~\ref{Complete-Thm-Wu}, it holds that
\begin{equation}
\label{eq-4.1}
\inf\Omega=
\begin{cases}
\langle \hat{\mu},\hat{\mu}-\xi(\Omega)\rangle, & \xi(\Omega)\in \mathscr{S}(\Omega), \\
\left\langle \frac{1+\xi(\Omega)}{2}, \frac{1-\xi(\Omega)}{2}\right\rangle, & \xi(\Omega)\notin \mathscr{S}(\Omega),
\end{cases}
\end{equation}
where $\hat{\mu}=\inf\{\mu_{\alpha} \mid
\langle \mu_{\alpha},\mu_{\alpha}-\xi(\Omega)\rangle \in\Omega\}$.

Similarly to the proof of Theorem~\ref{Complete-Thm-Wu}, it can be verified that
\begin{equation}
\label{eq-4.2}
\sup\Omega
=
\begin{cases}
\langle \tilde{\mu}, \tilde{\mu}-\eta(\Omega)\rangle, & \eta(\Omega)\in \mathscr{S}(\Omega), \\
\langle 0, -\eta(\Omega)\rangle, & \eta(\Omega) \notin \mathscr{S}(\Omega) \text{ and } \eta(\Omega)\leq 0, \\
\langle \eta(\Omega), 0 \rangle, & \eta(\Omega) \notin \mathscr{S}(\Omega) \text{ and } \eta(\Omega)> 0,
\end{cases}
\end{equation}
where $\tilde{\mu}=\sup\{\mu_{\alpha} \mid \langle \mu_{\alpha},\mu_{\alpha}-\eta(\Omega)\rangle\in\Omega\}$.
\end{remark}

The operator `$\overline{\alpha}$' in Definition~\ref{Def-Int-Operations}
is not a negation on $\tilde{\mathbb{I}}$ under the order $\leq_{_{\text{XY}}}$.
In the following, by introducing a new operator `$\neg$' for IFVs, we show that
it is a strong negation on $(\tilde{\mathbb{I}}, \leq_{_{\text{XY}}})$ (see Corollary~\ref{co-strong-negation}).
Moreover, we will prove that $(\tilde{\mathbb{I}}, \wedge, \vee, \neg)$ is a Kleene algebra
(see Theorem~\ref{Kleene-algebra-Thm}).

For $\alpha=\langle \mu_\alpha,\nu_\alpha\rangle \in\tilde{\mathbb{I}}$, define
\begin{equation}\label{eq-4.3}
\neg\alpha
=\begin{cases}
\left\langle \frac{1-\mu_\alpha-\nu_\alpha}{2},\frac{1+\mu_\alpha-3\nu_\alpha}{2}\right\rangle, & \mu_\alpha>\nu_\alpha, \\
\left\langle \frac{1}{2}-\mu_\alpha,\frac{1}{2}-\nu_\alpha\right\rangle, & \mu_\alpha=\nu_\alpha, \\
\left\langle \frac{1+\nu_\alpha-3\mu_\alpha}{2},\frac{1-\mu_\alpha-\nu_\alpha}{2}\right\rangle, & \mu_\alpha< \nu_\alpha.
\end{cases}
\end{equation}
Figure~\ref{Fig-2} is the geometric description of $\neg\alpha$.
\begin{figure}[H]
\begin{center}
\scalebox{0.55}{\includegraphics{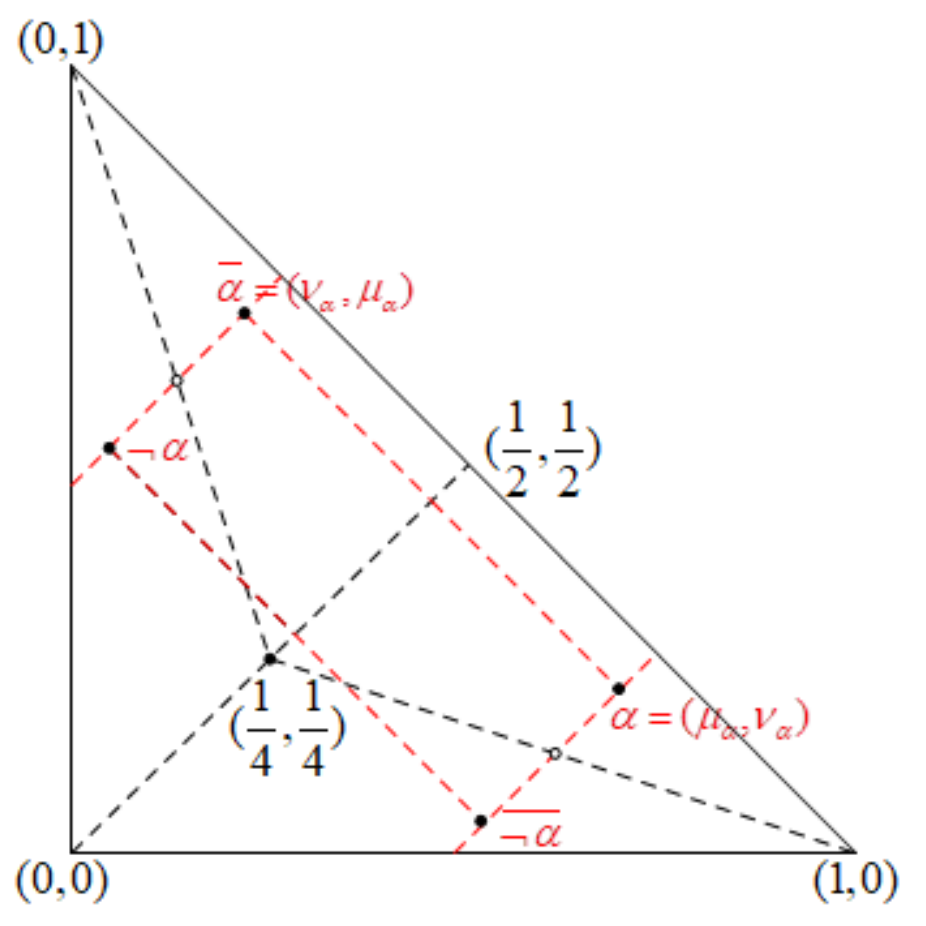}}
\end{center}
\caption{Geometrical interpretation of $\neg\alpha$}
\label{Fig-2}
\end{figure}

\begin{theorem}\label{Negation-Thm}
For $\alpha=\langle \mu_\alpha, \nu_\alpha\rangle$,
$\beta=\langle \mu_\beta, \nu_\beta\rangle\in\tilde{\mathbb{I}}$, we have

(1) $\neg\langle 1, 0\rangle =\langle 0, 1\rangle$ and $\neg \langle0, 1\rangle =\langle 1, 0\rangle$.

(2) $s(\neg\alpha)=-s(\alpha)$ and
   $$
  h(\alpha)=\begin{cases} 1-2\nu_\alpha, &\mu_\alpha > \nu_\alpha,\\
        1-(\mu_\alpha + \nu_\alpha), &\mu_\alpha= \nu_\alpha,\\
        1-2\mu_\alpha, & \mu_\alpha < \nu_\alpha.
        \end{cases}
   $$

(3) $\neg(\neg\alpha)=\alpha$.

(4) $\neg(\alpha\wedge\beta)=(\neg\alpha)\vee(\neg\beta)$.

(5) $\neg(\alpha\vee\beta)=(\neg\alpha)\wedge(\neg\beta)$.
\end{theorem}

\begin{proof}
(1) and (2) follow directly from formula~\eqref{eq-4.3}.

(3) It can be shown that $\neg(\neg\alpha)=\alpha$.
  \begin{itemize}
    \item If $\mu_\alpha=\nu_\alpha$, then $\neg(\neg\alpha)=
    \neg\big\langle\frac{1}{2}-\mu_\alpha, \frac{1}{2}-\nu_\alpha\big\rangle
    =\langle \mu_\alpha, \nu_\alpha\rangle =\alpha$.

    \item If $\mu_\alpha>\nu_\alpha $, then, by formula~\eqref{eq-4.3}, $\neg\alpha=\big\langle
     \frac{1-\mu_\alpha-\nu_\alpha}{2},\frac{1+\mu_\alpha-3\nu_\alpha}{2}\big\rangle$.
        This, together with $\frac{1-\mu_\alpha-\nu_\alpha}{2}-\frac{1+\mu_\alpha-3\nu_\alpha}{2}=
        -(\mu_\alpha-\nu_\alpha)$ and the formula~\eqref{eq-4.3}, implies that
        \begin{align*}
        \neg(\neg\alpha)
        &=\Big\langle \frac{1}{2}\Big[1+\frac{1+\mu_\alpha-3\nu_\alpha}{2}
        -\frac{3}{2}(1-\mu_\alpha-\nu_\alpha )\Big],\\
        &\quad \quad \frac{1}{2}\Big[1-\frac{1-\mu_\alpha-\nu_\alpha}{2}
        -\frac{1+\mu_\alpha-3\nu_\alpha}{2}\Big]\Big\rangle \\
        &=\langle \mu_\alpha, \nu_\alpha\rangle =\alpha.
        \end{align*}

  \item If $\mu_\alpha< \nu_\alpha $, then, similarly to the above proof, $\neg(\neg\alpha)=\alpha$.
  \end{itemize}

(4) The equality $\neg(\alpha\wedge\beta)=(\neg\alpha)\vee(\neg\beta)$ holds trivially
when $\alpha=\beta$. Without loss of generality, assume that $\alpha<_{_{\text{XY}}}\beta$.
We consider the following cases:
   \begin{itemize}
    \item Let $s(\alpha)=s(\beta)$ and $h(\alpha)<h(\beta)$. This means that
    $\mu_{\beta}>\mu_{\alpha}$ and $\nu_{\beta}>\nu_{\alpha}$.
       \begin{itemize}
       \item If $s(\alpha)=s(\beta)=0$, then $\neg(\alpha\wedge\beta)=\neg\alpha
       =\big\langle \frac{1}{2}-\mu_\alpha, \frac{1}{2}-\nu_\alpha\big\rangle$ and
       $\neg\beta=\big\langle \frac{1}{2}-\mu_\beta, \frac{1}{2}-\nu_\beta\big\rangle$.
       Observe that $s(\neg\alpha)=s(\neg\beta)$ and $h(\neg\alpha)=1-h(\alpha)>1-h(\beta)
       =h(\neg\beta)$, i.e., $\neg \alpha>_{_{\text{XY}}} \neg \beta$. Thus,
       $(\neg \alpha) \vee (\neg \beta)=\neg\alpha=\neg (\alpha \wedge \beta)$.

       \item If $s(\alpha)=s(\beta)>0$, then, by formula~\eqref{eq-4.3}, we have
       \begin{equation}
       \label{eq-4.4}
       \neg(\alpha\wedge\beta)=\neg \alpha =
       \left\langle \frac{1-\mu_\alpha-\nu_\alpha}{2},
       \frac{1+\mu_\alpha-3\nu_\alpha}{2}\right\rangle,
       \end{equation}
       and
       \begin{equation}
       \label{eq-4.5}
       \neg\beta=\left\langle \frac{1-\mu_\beta-\nu_\beta}{2},
       \frac{1+\mu_\beta-3\nu_\beta}{2}\right\rangle,
       \end{equation}
       This observation shows that $s(\neg\alpha)=-s(\alpha)=-s(\beta)=s(\neg \beta)$ and
       $h(\neg\alpha)=1-2\nu_\alpha>1-2\nu_\beta=h(\neg\beta)$. Hence, $\neg\alpha>_{_{\text{XY}}}\neg\beta$. This, together with the formula~\eqref{eq-4.4},
       implies that
       $$
       (\neg\alpha)\vee(\neg\beta)=\neg\alpha=\neg(\alpha\wedge\beta).
       $$

       \item If $s(\alpha)=s(\beta)<0$, then, similarly to the above proof,  $\neg(\alpha\wedge\beta)=(\neg\alpha)\vee(\neg\beta)$.
       \end{itemize}

  \item Let $s(\alpha)<s(\beta)$. Then, $s(\neg\alpha)=-s(\alpha)>-s(\beta)=s(\neg\beta)$,
  i.e., $\neg\alpha>_{_{\text{XY}}}\neg\beta$. Thus, $\neg(\alpha\wedge\beta)=
  (\neg\alpha)\vee(\beta)$.
     \end{itemize}

  (5) It follows from (3) and (4) that
  $\neg((\neg\alpha)\wedge(\neg\beta))=(\neg(\neg\alpha))\vee(\neg(\neg\beta))=\alpha\vee\beta$.
  Thus, $(\neg\alpha)\wedge(\neg\beta)=\neg(\alpha\vee\beta)$.
\end{proof}

\begin{corollary}\label{co-strong-negation}
The operator $\neg$ defined by formula~\eqref{eq-4.3}
is a strong negation on $(\tilde{\mathbb{I}}, \leq_{_{\text{XY}}})$.
\end{corollary}

\begin{proof}
It follows directly from Proposition~\ref{Pro-order-Xu} and Theorem~\ref{Negation-Thm}.
\end{proof}


\begin{example}
Take  $\alpha=\langle 0, 0\rangle$ and $\beta=\langle\frac{1}{2}, \frac{1}{2}\rangle$.
It can be verified that $\alpha\cap\beta=\langle 0, \frac{1}{2}\rangle$,
$\alpha\cup \beta=\langle\frac{1}{2}, 0\rangle$, and $\neg\alpha=\beta$.
Therefore, $\neg(\alpha \cap \beta)=\langle\frac{3}{4}, \frac{1}{4}\rangle$ and
$(\neg\alpha)\cup(\neg\beta)=\beta \cup\alpha=\langle\frac{1}{2}, 0\rangle$.
This means that $\neg(\alpha\cap\beta)\neq (\neg\alpha)\cup(\neg \beta)$.
\end{example}

\begin{theorem}
\label{Kleene-algebra-Thm}
$(\tilde{\mathbb{I}}, \wedge, \vee, \neg)$ is a Kleene algebra,
where $\wedge$ and $\vee$ are infimum and supremum operations
under the order $\leq_{_\text{XY}}$, respectively.
\end{theorem}

\begin{proof}
By Proposition \ref{Pro-order-Xu}, $(\tilde{\mathbb{I}}, \wedge, \vee)$
is a distributive lattice since $\leq_{_{\text{XY}}}$ is a linear order. This, together
with Corollary~\ref{co-strong-negation},
implies that $(\tilde{\mathbb{I}}, \wedge, \vee, \neg)$ is a De Morgan algebra. So, we only need to show that
$(\tilde{\mathbb{I}}, \wedge, \vee, \neg)$ satisfies Kleene's inequality.

For any $\alpha=\langle\mu_{\alpha}, \nu_{\alpha}\rangle$,
$\beta=\langle\mu_{\beta}, \nu_{\beta}\rangle
\in\tilde{\mathbb{I}}$, we consider the following cases:

(1)  If $\alpha \leq_{_{\text{XY}}} \beta$, then $\alpha \wedge (\neg \alpha)
\leq_{_{\text{XY}}} \alpha \leq_{_{\text{XY}}} \beta \leq_{_{\text{XY}}} \beta \vee (\neg \beta)$.

(2) If $\alpha >_{_{\text{XY}}} \beta$, then we have the following:

2.1) If $s(\alpha)>s(\beta)$, it follows from Theorem~\ref{Negation-Thm} (2) that
$s(\alpha \wedge (\neg \alpha))=\min\{s(\alpha), -s(\alpha)\} \leq 0$ and
$s(\beta\vee (\neg \beta))=\max\{s(\beta), -s(\beta)\}\geq 0$. Thus,
\begin{itemize}
\item if $s(\alpha) > 0$, then $s(\alpha \wedge (\neg \alpha))=-s(\alpha)<0
\leq s(\beta \vee (\neg \beta))$, which implies that $\alpha \wedge (\neg \alpha)
<_{_\text{XY}} \beta \vee (\neg \beta)$;

\item if $s(\alpha)=0>s(\beta)$, then $s(\alpha \wedge (\neg \alpha))=0<-s(\beta)=
s(\beta \vee (\neg \beta))$, which implies that
$\alpha \wedge (\neg \alpha)<_{_\text{XY}} \beta \vee (\neg \beta)$;

\item if $s(\alpha)<0$, then $s(\alpha\wedge(\neg \alpha))=
s(\alpha)<0\leq s(\beta\vee(\neg \beta))$, which implies that
$\alpha \wedge(\neg \alpha)<_{_{\text{XY}}}\beta\vee(\neg \beta)$.
\end{itemize}

Therefore, $\alpha\wedge(\neg \alpha)<_{_\text{XY}}
\beta \vee(\neg \beta)$.

2.2) If $s(\alpha)=s(\beta)$ and $h(\alpha)>h(\beta)$,
it follows from Theorem~\ref{Negation-Thm} (2) that
$s(\alpha\wedge (\neg\alpha))=\min\{s(\alpha), -s(\alpha)\}=
-\max\{s(\alpha), -s(\alpha)\}=-\max\{s(\beta), -s(\beta)\}
=-s(\beta\vee(\neg \beta))$. Thus,
\begin{itemize}
\item if $s(\alpha) \neq 0$, then $s(\alpha \wedge(\neg \alpha))
=\min \{s(\alpha), -s(\alpha)\}<0< \max\{s(\beta), -s(\beta)\}
=s(\beta \vee(\neg \beta))$, which implies that $\alpha \wedge(\neg\alpha)
<_{_\text{XY}} \beta \vee (\neg \beta)$;

\item if $s(\alpha)=s(\beta)=0$, i.e., $\mu_{\alpha}=\nu_{\alpha}$ and
$\mu_{\beta}=\nu_{\beta}$, then $\neg\alpha=\left\langle\frac{1}{2}-\mu_{\alpha},
\frac{1}{2}-\mu_{\alpha}\right\rangle$ and $\neg\beta=\left\langle\frac{1}{2}-\mu_{\beta},
\frac{1}{2}-\mu_{\beta}\right\rangle$. This implies that
\begin{align*}
\alpha\wedge(\neg \alpha)&=\left\langle \min \left\{\mu_{\alpha},
\frac{1}{2}-\mu_{\alpha} \right\},
\min\left\{\mu_{\alpha}, \frac{1}{2}-\mu_{\alpha} \right\} \right\rangle\\
&\leq_{_{\text{XY}}} \left\langle\frac{1}{4} , \frac{1}{4} \right\rangle,
\end{align*}
and
\begin{align*}
\beta \vee (\neg \beta)&=\left\langle \max \left\{\mu_{\beta},
\frac{1}{2}-\mu_{\beta} \right\}, \max \left\{\mu_{\beta},
\frac{1}{2}-\mu_{\beta} \right\} \right\rangle\\
&\geq_{_{\text{XY}}} \left\langle\frac{1}{4} , \frac{1}{4} \right\rangle.
\end{align*}
Therefore, $\alpha \wedge (\neg \alpha)
\leq _{_\text{XY}} \beta \vee (\neg \beta)$.
\end{itemize}

Summing up the above shows that
$\alpha\wedge(\neg \alpha)\leq_{_\text{XY}}\beta\vee (\neg \beta)$.
\end{proof}

By Theorems~\ref{th-isomorphism}, \ref{Complete-Thm-Wu},
and \ref{Kleene-algebra-Thm}, we have the following results.

\begin{theorem}
\label{Thm-ZX-Comp}
$(\tilde{\mathbb{I}}, \leq _{_\mathrm{ZX}})$ is a complete lattice.
\end{theorem}

\begin{theorem}
\label{Thm-ZX-Kleene}
$(\tilde{\mathbb{I}}, \wedge, \vee, \tilde{\Gamma}^{-1} \circ \neg \circ \tilde{\Gamma})$
is a Kleene algebra, where $\wedge$ and $\vee$ are infimum and supremum operations under
the order $\leq _{_\mathrm{ZX}}$, respectively, and $\tilde{\Gamma}$ is defined by the formula~\eqref{eq-3.2}.
\end{theorem}

\begin{remark}
Theorems~\ref{Negation-Thm} and \ref{Thm-ZX-Kleene} and Corollary~\ref{co-strong-negation} show that
the strong negation operators $\neg$ and $\tilde{\Gamma}^{-1} \circ \neg \circ \tilde{\Gamma}$
are new operators on both IFVs and IFSs. This partially answers Problem~\ref{Prob-1}.
\end{remark}

\begin{remark}
Through the linear orders $\leq _{_\mathrm{XY}}$ and $\leq _{_\mathrm{ZX}}$, and
the strong negation operators $\neg$ and $\tilde{\Gamma}^{-1} \circ \neg \circ \tilde{\Gamma}$,
we obtain in Theorems~\ref{Complete-Thm-Wu}, \ref{Kleene-algebra-Thm},
\ref{Thm-ZX-Comp}, and \ref{Thm-ZX-Kleene} some IF algebraic properties from exactly different perspectives, which partially answers Problem~\ref{Prob-2}.
\end{remark}

\begin{remark}
Because both $(\tilde{\mathbb{I}}, \leq_{_\mathrm{XY}})$ and $(\tilde{\mathbb{I}}, \leq _{_\mathrm{ZX}})$
are complete lattice, we can establish the decomposition theorem and Zadeh's extension principle
for IFSs as follows. For an IFS $I\in \mathrm{IFS}(X)$:

(Decomposition Theorem) For every $x\in X$,
$$
I(x)=\bigvee\{\alpha\in \tilde{\mathbb{I}}
\mid x\in I_{\alpha}\},
$$
where $I(x)=\langle \mu_{_{I}}(x), \nu_{_{I}}(x)\rangle$,
$I_{\alpha}=\{z\in X \mid I(z)\geq_{_\mathrm{XY}}\alpha\}$ and
$\vee$ is the supremum under the linear order $\leq_{_\mathrm{XY}}$.

(Zadeh's Extension Principle) Let $X$ and $Y$ be two nonempty sets and
$f: X\rightarrow Y$ be a mapping from $X$ to $Y$. Define a mapping
$\mathfrak{F}: \tilde{\mathbb{I}}^{X} \rightarrow \tilde{\mathbb{I}}^{Y}$ by
\begin{align*}
\mathfrak{F}: \tilde{\mathbb{I}}^{X} &\rightarrow \tilde{\mathbb{I}}^{Y}\\
I& \mapsto \mathfrak{F}(I)(y)=
\begin{cases}
\langle 0, 1\rangle, & f^{-1}(\{y\})=\varnothing, \\
\bigvee_{x\in f^{-1}(\{y\})}I(x), & f^{-1}(\{y\})\neq\varnothing,
\end{cases}
\end{align*}
which is called {\it Zadeh's extension mapping} of $f$ in the sense of
IFSs.
\end{remark}

\section{Topological structure of $(\tilde{\mathbb{I}}, \leq_{_{\text{XY}}})$
under the order topology}\label{Sec-5}

This section is devoted to investigating the topological structure of $(\tilde{\mathbb{I}}, \leq_{_{\text{XY}}})$
under the order topology. In particular, we show that the space $(\tilde{\mathbb{I}}, \leq_{_{\text{XY}}})$
under the order topology induced by the linear order $<_{_{\text{XY}}}$
is not separable and metrizable but compact and connected. We also demonstrate
that $(\tilde{\mathbb{I}}, \leq_{_{\text{ZX}}})$ has the same
topological structure as $(\tilde{\mathbb{I}}, \leq_{_{\text{XY}}})$ since
they are homeomorphic.

\begin{definition}[{\textrm{\protect\cite{Eng1977}}}]
A topological space with a countable dense subset is called \textit{separable}.
\end{definition}

\begin{definition}[{\textrm{\protect\cite{Eng1977}}}]
A topological space $X$ is called a \textit{compact space} if $X$ is a Hausdorff space and
every open cover of $X$ has a finite subcover; namely, there
exists a finite set $\{s_{1}, s_{2}, \ldots, s_{n}\} \subset \mathscr{G}$
such that $U_{s_{1}}\cup U_{s_{2}}\cup \cdots \cup U_{s_{n}}=X$ for every open cover
$\{U_{s}\}_{s\in\mathscr{G}}$ of the space $X$.
\end{definition}

\begin{definition}[{\textrm{\protect\cite{Eng1977, M1975}}}]
Let $X$ be a topological space. A \textit{separation} of $X$ is a pair $(U,V)$
of disjoint nonempty open subsets of $X$ whose union is $X$. The space $X$ is
said to be \textit{connected} if there does not exist a separation of $X$.
\end{definition}

\begin{definition}[{\textrm{\protect\cite{Eng1977}}}]
Let $X$ be a topological space. It is called \textit{metrizable} if there is a metric
$\varrho$ on $X$ such that the topology induced by this metric
coincides with the original topology of $X$.
\end{definition}

\begin{lemma}
[{\textrm{\protect\cite[Problem~3.12.3 (a)]{Eng1977}, \cite{HK1910}}}]
\label{Compact-Cha}
A space $X$ with the topology induced by a linear order $\prec$ is
compact if and only if every subset A of $X$ has the least upper bound.
\end{lemma}

\begin{lemma}
[{\textrm{\protect\cite[Problem~6.3.2 (a)]{Eng1977}}}]
\label{Connect-Cha}
A space $X$  with the topology induced by a linear order $\prec$ is
connected if and only if it is a continuously ordered set by $\prec$.
\end{lemma}

\begin{proposition}
\label{Compact-Lattice}
The space $\tilde{\mathbb{I}}$ with the order topology induced
by the linear order $<_{_{\text{XY}}}$ defined in Definition~\ref{de-order(Xu)}
is compact.
\end{proposition}

\begin{proof}
It follows directly from Theorem \ref{Complete-Thm-Wu} and Lemma \ref{Compact-Cha}.
\end{proof}

\begin{proposition}
The space $\tilde{\mathbb{I}}$ with the order topology induced by the
linear order $<_{_{\text{XY}}}$ defined in Definition \ref{de-order(Xu)}
is connected.
\end{proposition}

\begin{proof}
Let $(D, E)$ be a cut of $\tilde{\mathbb{I}}$.

(1) We first show that $(D, E)$ is not a jump.

 Suppose, on the contrary, that $(D, E)$ is a jump. Then,
  $D$ and $E$ have the largest and the smallest elements,
 denoted by $\xi$ and $\eta$, respectively. In this case,
 $\xi <_{_{\text{XY}}} \eta $ since $(D, E)$ is a cut of
 $\tilde{\mathbb{I}}$, $\xi \in D$ and $\eta \in E$. This implies that
 $X=D\cup E \subset (\leftarrow, \xi] \cup [\eta, \rightarrow)$. Thus,
 $(\xi, \eta)= X\backslash (\leftarrow, \xi]\cup [\eta, \rightarrow)=\varnothing$.
 This is impossible since $\big\langle \frac{1}{4}(s(\xi)+s(\eta)+h(\xi)+h(\eta)),
 \frac{1}{4}(h(\xi)+h(\eta)-s(\xi)-s(\eta))\big\rangle \in
 (\xi, \eta)$. Therefore, $(D, E)$ is not a jump.

(2) We then show that $(D, E)$ is not a gap.

Suppose, on the contrary, that $(D, E)$ is a gap. Then,
$D$ and $E$ have no largest and smallest elements, respectively.
It follows from Theorem~\ref{Complete-Thm-Wu} that $\sup D$ exists. It can be verified that
\begin{enumerate}[(i)]
\item $\sup D \notin D $. This, together with $D\cup E= \tilde{\mathbb{I}}$, implies that  $\sup D\in E$.

\item For any $\beta\in E$, since $\beta$ is an upper bound of $D$, we get that
$\sup D\leq_{_{\text{XY}}}\beta$.
\end{enumerate}

Therefore, $\sup D$ is the smallest element of \textit{E}. This is a contradiction. Hence, $(D, E)$ is not a gap.

Summing up (1) and (2) shows that $\tilde{\mathbb{I}}$ is a continuously ordered set.
Hence, $\tilde{\mathbb{I}}$ is connected by applying Lemma~\ref{Connect-Cha}.
\end{proof}

\begin{proposition}
\label{Not-Separable-Pro}
The space $\tilde{\mathbb{I}}$ with the order topology induced by the linear order
$<_{_{\text{XY}}}$, defined in Definition~\ref{de-order(Xu)}, is not separable.
\end{proposition}

\begin{proof}
Consider any dense subset $D$ of $\tilde{\mathbb{I}}$. By
Proposition~\ref{Pro-order-Xu},
$I_{\gamma}=\left(\langle \gamma, 0\rangle, \left\langle \frac{1+\gamma}{2},
\frac{1-\gamma}{2}\right\rangle\right)=
\{\langle\mu, \mu-\gamma\rangle \in \tilde{\mathbb{I}} \mid \gamma<_{_{\text{XY}}}
\mu<_{_{\text{XY}}} \frac{1+\gamma}{2}\}$ is a nonempty open subset of $\tilde{\mathbb{I}}$
for any $\gamma\in (0, 1)$. Thus, $D\cap I_{\gamma}\neq \varnothing$. For any
$\gamma\in (0, 1)$, choose a point $\alpha_{\gamma}\in D\cap I_{\gamma}$.
Since $I_{\gamma}\cap I_{\gamma^{\prime}}=\varnothing$ for any $\gamma$,
$\gamma^{\prime}\in (0, 1)$ with $\gamma \neq \gamma^{\prime}$, we have that
$\aleph_{0}< 2^{\aleph_{0}}=|(0, 1)|=|\{\alpha_{\gamma} \mid \gamma
\in (0, 1)\}|\leq |D|$. Therefore, $\tilde{\mathbb{I}}$ is not separable.
\end{proof}

\begin{corollary}
\label{Corollary-not-Home}
The space $\tilde{\mathbb{I}}$ with the order topology induced by the linear
order $<_{_{\text{XY}}}$, defined in Definition~\ref{de-order(Xu)}, is not
homeomorphic to $[0, 1]$ with the natural topology.
\end{corollary}

\begin{proof}
It follows directly from Proposition~\ref{Not-Separable-Pro} and the fact that $[0, 1]$
is separable under the natural
topology.
\end{proof}

\begin{corollary}
\label{Corollary-not-Metri}
The space $\tilde{\mathbb{I}}$ with the order topology induced by
the linear order $<_{_{\text{XY}}}$, defined in Definition~\ref{de-order(Xu)},
is not metrizable.
\end{corollary}

\begin{proof}
Suppose, on the contrary, that  $\tilde{\mathbb{I}}$ is metrizable. Then,
there is a metric $\varrho$ on $\tilde{\mathbb{I}}$ such that the
topology induced by this metric and the order topology coincide.
By Proposition~\ref{Compact-Lattice}, $(X, \varrho)$ is
a compact metric space. Thus, it is separable, which contradicts
Proposition~\ref{Not-Separable-Pro}. Therefore, $\tilde{\mathbb{I}}$ is not metrizable.
\end{proof}

\begin{corollary}
The spaces $([0, 1], \leq)$ and $(\tilde{\mathbb{I}}, \leq_{_{\text{XY}}})$ are not isomorphic.
\end{corollary}

\begin{proof}
Suppose, on the contrary, that $\left([0, 1], \leq\right)$ and
${(\tilde{\mathbb{I}}, \leq_{_{\text{XY}}})}$ are isomorphic. In this case, there exists a bijection
$f: [0, 1]\rightarrow \tilde{\mathbb{I}}$ such that both $f$ and $f^{-1}$ are order preserving.
Take $\mathscr{A}= f([0, 1]\cap \mathbb{Q})$. For any
$ \alpha,\beta \in \tilde{\mathbb{I}} $ with $\alpha<_{_{\text{XY}}}\beta $, we have that
$f^{-1}(\alpha)< f^{-1}(\beta)$. It is obvious that $(f^{-1}(\alpha), f^{-1}(\beta))\cap \mathbb{Q}\neq\varnothing$.
This implies that $\varnothing\neq f((f^{-1}(\alpha), f^{-1}(\beta))\cap\mathbb{Q})=(\alpha, \beta)
\cap \mathscr{A}$. Since $\mathscr{A}$ is countable, $\tilde{\mathbb{I}}$ with the order topology
is separable. This contradicts Proposition~\ref{Not-Separable-Pro}. Therefore, $([0, 1], \leq)$
and $(\tilde{\mathbb{I}}, \leq_{_{\text{XY}}})$ are not isomorphic.
\end{proof}

\begin{remark}
(1) By formula~\eqref{eq-3.1}, the space $\tilde{\mathbb{I}}$
with the order topology induced by the linear order $<_{_{\text{XY}}}$ is
homeomorphic to $\tilde{\mathbb{I}}$ with the order topology induced by the
linear order $<_{_{\mathrm{ZX}}}$. Thus,
they have the same topological structure.

(2) Propositions~\ref{Compact-Lattice}--\ref{Not-Separable-Pro} and
Corollaries~\ref{Corollary-not-Home} and \ref{Corollary-not-Metri}
partially answer Problem~\ref{Prob-3}.
\end{remark}

\section{A new look at q-rung orthopair fuzzy sets}\label{Sec-6}

The concept of Pythagorean fuzzy set (PFS) was introduced by Yager~\cite{Yager2014} as follows:
\begin{definition}[{\textrm{\protect\cite{Yager2014}}}]
Let $X$ be the universe of discourse. A \textit{Pythagorean fuzzy set} (PFS)
$P$ in $X$ is defined as an object in the following form:
\begin{equation}
P=\{\langle x, \mu_{_P}(x), \nu_{_P}(x) \rangle \mid x \in X \},
\end{equation}
where the functions
$$
\mu_{_P}: X\rightarrow [0, 1],
$$
and
$$
\nu_{_P}: X\rightarrow [0, 1],
$$
transform the {\it degree of membership} and the \textit{degree of non-membership} of
the element $x \in X$ to the set $P$, respectively, and for every $x \in X$,
\begin{equation}
(\mu_{_P}(x))^{2}+(\nu_{_P}(x))^{2} \leq 1.
\end{equation}
The \textit{indeterminacy degree} $\pi_{_P}(x)$ of element $x$ belonging to the
PFS $P$ is defined by
$$
\pi_{_P}(x)=\sqrt{1-(\mu_{_P}(x))^{2}-(\nu_{_P}(x))^{2}}.
$$
Zhang and Xu~\cite{ZX2014} called $\langle \mu_{_P}(x), \nu_{_P}(x)\rangle$ a
\textit{Pythagorean fuzzy number} (PFN). Let $\widetilde{\mathbb{P}}$ denote the set of all PFNs,
i.e., $\widetilde{\mathbb{P}}=\{\langle \mu, \nu\rangle \in [0, 1]^{2} \mid \mu^{2}+\nu^{2} \leq 1\}$.
\end{definition}

Recently, Yager et al.~\cite{Yager2017,YagerA2017} introduced the concept of
q-rung orthopair fuzzy sets.

\begin{definition}[{\textrm{\protect\cite{Yager2017}}}]
Let $X$ be the universe of discourse. A \textit{q-rung orthopair fuzzy set}
(q-ROFS) $Q$ in $X$ is defined as an object in the following form:
\begin{equation}
Q=\{\langle x, \mu_{_Q}(x), \nu_{_Q}(x)\rangle  \mid  x \in X\},
\end{equation}
where the functions
$$
\mu_{_Q}: X\rightarrow [0, 1],
$$
and
$$
\nu_{_Q}: X\rightarrow [0, 1],
$$
transform the \textit{degree of membership} and the
\textit{degree of non-membership} of the element $x \in X$ to the set $Q$,
respectively, and for every $x \in X$,
\begin{equation}
(\mu_{_Q}(x))^{q}+(\nu_{_Q}(x))^{q} \leq 1, \quad (q \geq 1).
\end{equation}
The \textit{indeterminacy degree} $\pi_{_Q}(x)$ of element $x$ belonging to the
q-ROFS $Q$ is defined by $\pi_{_Q}(x)=\sqrt[q]{1-(\mu_{_Q}(x))^{q}-(\nu_{_Q}(x))^{q}}$.
Yager~\cite{Yager2017} called $\langle \mu_{_Q}(x), \nu_{_Q}(x)\rangle$
a \textit{q-rung orthopair fuzzy number} (q-ROFN). Let $\widetilde{\mathbb{Q}}$ denote
the set of all q-ROFNs, i.e.,
$\widetilde{\mathbb{Q}}=\{\langle\mu, \nu\rangle \in [0, 1]^{2} \mid \mu^{q}+\nu^{q} \leq 1\}$.
\end{definition}

It is clear that a q-ROFS reduces to an IFS (resp. PFS) when $q=1$ (resp. $q=2$).
When $q=3$, Senapati and Yager \cite{SY2020} called the corresponding q-rung orthopair
fuzzy sets as \textit{Fermatean fuzzy sets}(FFSs).

Additionally, for a q-ROFN $\alpha=\langle\mu_{\alpha}, \nu_{\alpha}\rangle \in \widetilde{\mathbb{Q}}$,
$s_{_{\text{LW}}}(\alpha)=(\mu_{\alpha})^{q}-(\nu_{\alpha})^{q}$, and
$h_{_{\text{LW}}}(\alpha)=(\mu_{\alpha})^{q}+(\nu_{\alpha})^{q}$ were called the
\textit{score value} and the \textit{accuracy value} of $\alpha$, respectively, by Liu and Wang~\cite{LW2018}.

Based on score values and accuracy values, a comparison method was given by Liu and Wang~\cite{LW2018}.

\begin{definition}[{\textrm{\protect\cite{LW2018}}}]
\label{qROFNs-LW-order-Def}
Let $\alpha_{1}$ and $\alpha_{2}$ be two q-ROFNs.
\begin{itemize}
\item If $s_{_{\text{LW}}}(\alpha_{1})<s_{_{\text{LW}}}(\alpha_{2})$,
then $\alpha_{1}$ is smaller than $\alpha_{2}$, denoted by
$\alpha_{1} <_{_{\text{LW}}} \alpha_{2}$.
\item If $s_{_{\text{LW}}}(\alpha_{1})=s_{_{\text{LW}}}(\alpha_{2})$, then
\begin{itemize}
\item[$-$] if $h_{_{\text{LW}}}(\alpha_{1})=h_{_{\text{LW}}}(\alpha_{2})$,
then $\alpha_{1}=\alpha_{2}$;
\item[$-$] if $h_{_{\text{LW}}}(\alpha_{1})< h_{_{\text{LW}}}(\alpha_{2})$,
then $\alpha_{1}<_{_{\text{LW}}} \alpha_{2}$.
\end{itemize}
\end{itemize}
If $\alpha_{1} <_{_{\text{LW}}} \alpha_{2}$ or
$\alpha_{1}=\alpha_{2}$, denote it by
$\alpha_{1}\leq_{_{\text{LW}}}\alpha_{2}$.
\end{definition}

Observe that the score function and the accuracy function, introduced in Definition~\ref{qROFNs-LW-order-Def},
are compressed after squaring each difference value. Hence, to rank any pair of
q-ROFNs, Xing et al.~\cite{XZZW2019} proposed the modified score function and
accuracy function for q-ROFNs, respectively, as follows.

\begin{definition}[{\textrm{\protect\cite[Definitions~5 and 6]{XZZW2019}}}]
Let $\alpha=\langle\mu, \nu\rangle$ be a q-ROFN and $q \geq 1$. The \textit{modified score function}
$s_{_{{\text{X}}}}(\alpha)$ and the \textit{modified accuracy function}
$h_{_{{\text{X}}}}(\alpha)$ of $\alpha$ are defined as follows:
$$
s_{_{{\text{X}}}}(\alpha)=
\begin{cases}
\sqrt[q]{\mu^{q}-\nu^{q}}, & \mu \geq \nu,\\
-\sqrt[q]{\nu^{q}-\mu^{q}}, & \mu < \nu,
\end{cases}
$$
and
$$
h_{_{{\text{X}}}}(\alpha)=\sqrt[q]{\mu^{q}+\nu^{q}}.
$$
\end{definition}

\begin{definition}[{\textrm{\protect\cite[Definitions~7]{XZZW2019}}}]
Let $\alpha_{1}$ and $\alpha_{1}$ be two q-ROFNs.
\begin{itemize}
\item If $s_{_{{\text{X}}}}(\alpha_{1}) < s_{_{{\text{X}}}}(\alpha_{2})$,
then $\alpha_{1}$ is smaller than $\alpha_{2}$, denoted by
$\alpha_{1}<_{{_{\text{X}}}} \alpha_{2}$.

\item If $s_{_{{\text{X}}}}(\alpha_{1})=s_{_{{\text{X}}}}(\alpha_{2})$, then
\begin{itemize}
\item[$-$] if $h_{_{{\text{X}}}}(\alpha_{1})=h_{_{{\text{X}}}}(\alpha_{2})$,
then $\alpha_{1}=\alpha_{2}$;
\item[$-$] if $h_{_{{\text{X}}}}(\alpha_{1})< h_{_{{\text{X}}}}(\alpha_{2})$,
then $\alpha_{1} <_{_{{\text{X}}}} \alpha_{2}$.
\end{itemize}
\end{itemize}
If $\alpha_{1}<_{_{{\text{X}}}} \alpha_{2}$ or
$\alpha_{1}=\alpha_{2}$, denote it by
$\alpha_{1} \leq_{_{{\text{X}}}}\alpha_{2}$.
\end{definition}

\begin{proposition}\label{pr-order equi}
The orders $\leq_{_{\text{LW}}}$ and $\leq_{_{{\text{X}}}}$ are equivalent, i.e.,
for $\alpha_{1}$, $\alpha_{2} \in \widetilde{\mathbb{Q}}$, $\alpha_{1}\leq_{_{\text{LW}}} \alpha_{2}$
if and only if $\alpha_{1}\leq_{_{{\text{X}}}}\alpha_{2}$.
\end{proposition}

\begin{proof}
It follows directly that $s_{_{\text{LW}}}(\alpha_{1})< s_{_{\text{LW}}}(\alpha_{2})
\Leftrightarrow s_{_{{\text{X}}}}(\alpha_{1})<s_{_{{\text{X}}}}(\alpha_{2})$,
$s_{_{\text{LW}}}(\alpha_{1})=s_{_{\text{LW}}}(\alpha_{2})\Leftrightarrow
s_{_{{\text{X}}}}(\alpha_{1})=s_{_{{\text{X}}}}(\alpha_{2})$,
$h_{_{\text{LW}}}(\alpha_{1})< h_{_{\text{LW}}}(\alpha_{2}) \Leftrightarrow
h_{_{{\text{X}}}}(\alpha_{1}) < h_{_{{\text{X}}}}(\alpha_{2})$, and $h_{_{\text{LW}}}(\alpha_{1})=h_{_{\text{LW}}}(\alpha_{2}) \Leftrightarrow h_{_{{\text{X}}}}(\alpha_{1})=h_{_{{\text{X}}}}(\alpha_{2})$.
\end{proof}

Fix $q \geq 1$. Define a mapping $\Gamma: \widetilde{\mathbb{Q}}\rightarrow \widetilde{\mathbb{I}}$ by
\begin{equation}\label{eq-6.5}
\begin{split}
\Gamma: \widetilde{\mathbb{Q}}&\rightarrow \widetilde{\mathbb{I}}, \\
\langle \mu, \nu\rangle &\mapsto \langle\mu^{q}, \nu^{q}\rangle.
\end{split}
\end{equation}
It can be verified that $\Gamma$ is a bijection. Every partial order `$\preceq_{_{\widetilde{\mathbb{I}}}}$'
on $\widetilde{\mathbb{I}}$ can induce a partial order $\preceq_{_{\widetilde{\mathbb{Q}}}}$,
defined by $\alpha \preceq_{_{\widetilde{\mathbb{Q}}}} \beta$ ($\alpha, \beta \in \widetilde{\mathbb{Q}}$)
if and only if $\Gamma(\alpha) \preceq_{_{\widetilde{\mathbb{I}}}} \Gamma(\beta)$. Analogously,
every partial order on $\widetilde{\mathbb{Q}}$ can induce a partial order on $\widetilde{\mathbb{I}}$.

\begin{theorem}\label{Order-Xu=Order-LW}
The mapping defined by formula \eqref{eq-6.5} is an isomorphism between
$(\widetilde{\mathbb{I}}, \leq_{_{\text{XY}}})$ and
$(\widetilde{\mathbb{Q}}, \leq_{_{\text{LW}}})$.
\end{theorem}

\begin{proof}
Since $\Gamma$ is a bijection, it suffices to check that $\Gamma$
and $\Gamma^{-1}$ are order preserving. We only prove that $\Gamma$
is order preserving. Similarly, it can be shown that $\Gamma^{-1}$ is order-preserving.
For $\alpha$, $\beta \in \widetilde{\mathbb{Q}}$ with
$\alpha \leq_{_{\text{LW}}} \beta$, consider the following two cases:

(1) If $s_{_{\text{LW}}}(\alpha)< s_{_{\text{LW}}}(\beta)$, then
$s(\Gamma(\alpha))=(\mu_{\alpha})^{q}-(\nu_{\alpha})^{q}=s_{_{\text{LW}}}(\alpha)
<s_{_{\text{LW}}}(\beta)=(\mu_{\beta})^{q}-(\nu_{\beta})^{q}=s(\Gamma(\beta))$.
Thus, $\Gamma(\alpha)<_{_{\text{XY}}} \Gamma(\beta)$.

(2) If $s_{_{\text{LW}}}(\alpha) = s_{_{\text{LW}}}(\beta)$ and
$h_{_{\text{LW}}}(\alpha) < h_{_{\text{LW}}}(\beta)$, then
$s(\Gamma(\alpha))=s(\Gamma(\beta))$ and $h(\Gamma(\alpha))=(\mu_{\alpha})^{q}+(\nu_{\alpha})^{q}
=h_{_{\text{LW}}}(\alpha) < h_{_{\text{LW}}}(\beta)=(\mu_{\beta})^{q}
+(\nu_{\beta})^{q}=h(\Gamma(\beta))$. Thus, $\Gamma(\alpha)
<_{_{\text{XY}}} \Gamma(\beta)$.

Therefore, $\Gamma$ is order preserving.
\end{proof}

Applying the isomorphism defined by formula \eqref{eq-6.5},
similarly to Definition~\ref{de-order(ZX)}, we introduce a new linear order for q-ROFNs.
For $\alpha=(\mu_{\alpha}, \nu_{\alpha}) \in \widetilde{\mathbb{Q}}$,
define the similarity function $L_{_{\text{Wu}}}(\alpha)$ by
\begin{equation}
\begin{split}
L_{_{\text{Wu}}}(\alpha)&=\sqrt[q]{\frac{1-(\nu_{\alpha})^{q}}{(1-(\mu_{\alpha})^{q})
+(1-(\nu_{\alpha})^{q})}}\\
&=\sqrt[q]{\frac{1-(\nu_{\alpha})^{q}}{1+(\pi_{\alpha})^{q}}}.
\end{split}
\end{equation}

\begin{definition}
\label{qROFNs-Wu-order-Def}
Let $\alpha_{1}$ and $\alpha_{2}$ be two q-ROFNs.
\begin{itemize}
\item If $ L_{_{\text{Wu}}}(\alpha_{1})< L_{_{\text{Wu}}}(\alpha_{2})$,
then $\alpha_{1}$ is smaller than $\alpha_{2}$, denoted by $\alpha_{1} <_{_{\text{Wu}}} \alpha_{2} $.
\item If $ L_{_{\text{Wu}}}(\alpha_{1})= L_{_{\text{Wu}}}(\alpha_{2})$,
then
\begin{itemize}
\item[$-$] if $ h_{_{{\text{X}}}}(\alpha_{1})=h_{_{{\text{X}}}}(\alpha_{2})$,
then $\alpha_{1}= \alpha_{2}$;
\item[$-$] if $ h_{_{{\text{X}}}}(\alpha_{1})<h_{_{{\text{X}}}}(\alpha_{2})$,
then $\alpha_{1}  <_{_{\text{Wu}}} \alpha_{2}$.
\end{itemize}
\end{itemize}
If $\alpha_{1} <_{_{\text{Wu}}} \alpha_{2} $ or  $\alpha_{1}= \alpha_{2}$,
denote it by $\alpha_{1} \leq{_{_{\text{Wu}}}} \alpha_{2} $.
\end{definition}

Similarly to Theorem \ref{th-isomorphism}, we have the following result.

\begin{theorem}
\label{LW=Wu-Thm}
$(\widetilde{\mathbb{Q}}, \leq_{_{\text{LW}}})$ and
$(\widetilde{\mathbb{Q}}, \leq_{_{\text{Wu}}}) $ are isomorphism.
\end{theorem}

By Proposition \ref{pr-order equi} and Theorems~\ref{Order-Xu=Order-LW}
and \ref{LW=Wu-Thm}, we have the following results.

\begin{corollary}
$(\widetilde{\mathbb{Q}}, \leq_{_{\text{LW}}})$,
$(\widetilde{\mathbb{Q}}, \leq_{_{{\text{X}}}})$,
and  $(\widetilde{\mathbb{Q}}, \leq_{_{\text{Wu}}})$ are complete lattices.
\end{corollary}

\begin{corollary}
$(\widetilde{\mathbb{Q}}, \wedge, \vee, \widetilde{\neg})$ is a
Kleene algebra such that $\wedge$ and $\vee$ are infimum and supremum operations
under the order $\leq_{_{\text{LW}}}$, respectively, and $\widetilde{\neg}=\Gamma^{-1}
\circ \neg \circ \Gamma$, where $\neg$ and $\Gamma$ are given by formulas~\eqref{eq-4.3}
and \eqref{eq-6.5}, respectively.
\end{corollary}

\begin{corollary}
$(\widetilde{\mathbb{Q}}, \wedge, \vee, \widehat{\neg})$ is a Kleene algebra such that
$\wedge$ and $\vee$ are infimum and supremum operations
under the order $\leq_{_{\text{Wu}}}$, respectively, and $\widehat{\neg}=\Gamma^{-1}
\circ \widetilde{\Gamma}^{-1} \circ \neg \circ \widetilde{\Gamma} \circ \Gamma$,
where $\neg$, $\widetilde{\Gamma}$, and $\Gamma$ are given by formulas~\eqref{eq-4.3},
\eqref{eq-3.1}, and \eqref{eq-6.5}, respectively.
\end{corollary}

\begin{remark}
By formula \eqref{eq-4.3}, it can be verified that, for
$\alpha=\langle \mu_{\alpha},\nu_{\alpha}\rangle\in \widetilde{\mathbb{Q}}$,
\begin{equation}
\widetilde{\neg}\alpha=
\begin{cases}
\Big\langle \sqrt[q]{\frac{1-(\mu_{\alpha})^{q}-(\nu_{\alpha})^{q}}{2}},
\sqrt[q]{\frac{1+(\mu_{\alpha})^{q}-3(\nu_{\alpha})^{q}}{2}}\Big\rangle,
& \mu_{\alpha}>\nu_{\alpha}, \\
\Big\langle \sqrt[q]{\frac{1}{2}-(\mu_{\alpha})^{q}},
\sqrt[q]{\frac{1}{2}-(\nu_{\alpha})^{q}}\Big\rangle, & \mu_{\alpha}=\nu_{\alpha}, \\
\Big\langle \sqrt[q]{\frac{1+(\nu_{\alpha})^{q}-3(\mu_{\alpha})^{q}}{2}},
\sqrt[q]{\frac{1-(\mu_{\alpha})^{q}-(\nu_{\alpha})^{q}}{2}}\Big\rangle,
& \mu_{\alpha}<\nu_{\alpha}.
\end{cases}
\end{equation}
\end{remark}

\begin{corollary}
The space $\widetilde{\mathbb{Q}}$ with the order topology induced by
the linear order $<_{_{\text{LW}}}$, introduced in Definition~\ref{qROFNs-LW-order-Def},
is not metrizable and separable but compact and connected.
\end{corollary}

\begin{corollary}
The space $\widetilde{\mathbb{Q}}$ with the order topology induced by the
linear order $<_{_{\text{Wu}}}$, introduced in Definition~\ref{qROFNs-Wu-order-Def},
is not metrizable and separable but compact and connected.
\end{corollary}

Given a mapping $f: \widetilde{\mathbb{I}}^{n}\rightarrow \mathbb{I}$,
we obtain the induced mapping $\mathscr{S}(f)$ from
$\widetilde{\mathbb{Q}}^{n}$ to $\widetilde{\mathbb{Q}}$
as follows:
\begin{equation}
\label{eq-Trans}
\begin{split}
\mathscr{S}(f): \widetilde{\mathbb{Q}}^{n} &\rightarrow  \widetilde{\mathbb{Q}}, \\
(\langle \mu_{1}, \nu_{1}\rangle, \cdots, \langle\mu_{n}, \nu_{n}\rangle)
& \mapsto \sqrt[q]{f((\mu_{1}^{q}, \nu_{1}^{q}), \cdots,
(\mu_{n}^{q}, \nu_{n}^{q}))}.
\end{split}
\end{equation}
Denote $\Gamma^{(n)}=\underbrace{\Gamma\times\Gamma\times \cdots \times\Gamma}\limits_{n}$
(product mapping). Then $\mathscr{S}(f)=\Gamma^{-1}\circ f\circ \Gamma^{(n)}$.
By the definition of $\mathscr{S}(\_)$, the following theorem directly
follows from Proposition~\ref{pr-order equi}.
\begin{theorem}
\label{Equ-IFV-Q}
The following statements are equivalent:
\begin{enumerate}[(1)]
  \item The mapping $f: (\widetilde{\mathbb{I}}^{n}, \leq_{_{\text{XY}}})
\rightarrow (\widetilde{\mathbb{I}}, \leq_{_{\text{XY}}})$ is idempotent
(resp., increasing, commutative, or bounded);
  \item The mapping $\mathscr{S}(f):
(\widetilde{\mathbb{Q}}^{n}, \leq_{_{\text{LW}}}) \rightarrow  (\widetilde{\mathbb{Q}},
\leq_{_{\text{LW}}})$ is idempotent (resp., increasing, commutative, or bounded);
  \item The mapping $\mathscr{S}(f):
(\widetilde{\mathbb{Q}}^{n}, \leq_{_{{\text{X}}}}) \rightarrow  (\widetilde{\mathbb{Q}},
\leq_{_{{\text{X}}}})$ is idempotent (resp., increasing, commutative, or bounded).
\end{enumerate}
\end{theorem}

By applying the mapping $\mathscr{S}(\_)$, we obtain from Theorem~\ref{Equ-IFV-Q}
many known aggregation operations for q-ROFNs.

(1) Let $q=2$. If $f$ is an intuitionistic fuzzy weighted averaging
(IFWA) operator defined in \cite[Definition~3.3]{Xu2007}, then $\mathscr{S}(f)$ reduces to the PFWA operator
defined in~\cite[Definition~2.10]{Zhang2016}. If $f$ is an intuitionistic fuzzy ordered weighted
averaging (IFOWA) operator defined in \cite[Definition~3.4]{Xu2007}, then $\mathscr{S}(f)$
reduces to the PFOWA operator defined in \cite[Definition~2.11]{Zhang2016}.

(2) Let $q\geq 1$. If we take $f$ as $\oplus$, $\otimes$, $\lambda \alpha$, or $\alpha^{\lambda}$
in Definition~\ref{Def-Int-Operations}, then $\mathscr{S}(f)$
reduces to the formulas (4)--(7) in \cite{LW2018}, respectively. By \cite[Theorem~1.2.3]{XC2012},
we observe that \cite[Theorem~1]{LW2018} is true. Moreover,
if $f$ is an IFWA operator defined in \cite[Definition~3.3]{Xu2007},
then $\mathscr{S}(f)$ reduces to the q-ROFWA operator defined in \cite[Definition~5]{LW2018}.
If $f$ is an IFWG operator defined in~\cite[Definition~2]{XY2006}, then $\mathscr{S}(f)$
reduces to the q-ROFWG operator defined in \cite[Definition~6]{LW2018}.
By Theorem~\ref{Equ-IFV-Q}, \cite[Theorems~3--6]{XY2006}, and \cite[Theorems~3.4 and 3.5]{Xu2007},
we immediately obtain
\cite[Theorems~2--9]{LW2018}.

(3) Let $q\geq 1$. If  $f$ is an intuitionistic fuzzy Maclaurin symmetric mean (IFMSM)
operator in \cite[Definition~4]{QL2014}, then $\mathscr{S}(f)$ reduces to the q-rung orthopair
fuzzy MSM (q-ROFMSM) operator defined in \cite[Definition~11]{LCW2020}. By \cite[Properties~1--4]{LCW2020}
and Theorem~\ref{Equ-IFV-Q}, we observe that q-ROFMSM is idempotent, increasing, commutative,
and bounded. By~\cite[Theorem~1]{QL2014}, we have that,
for $(\langle \mu_{1}, \nu_{1}\rangle, \ldots, \langle\mu_{n}, \nu_{n}\rangle)\in \tilde{\mathbb{Q}}^{n}$,
\begin{align*}
&\mathscr{S}(\text{IFMSM})(\langle \mu_{1}, \nu_{1}\rangle, \cdots, \langle\mu_{n}, \nu_{n}\rangle)\\
=& \bigg\langle \bigg(\bigg(1-\bigg(\prod\limits_{1\leq i_1<\cdots <i_{k}\leq n}\bigg(1-\prod\limits_{j=1}^{k}(\mu_{i_j})^{q}\bigg)\bigg)^{\frac{1}{C_{n}^{k}}}
\bigg)^{\frac{1}{k}}\bigg)^{\frac{1}{q}},\\
& \bigg(1-\bigg(1-\bigg(\prod\limits_{1\leq i_1<\cdots <i_{k}\leq n}\bigg(1-\prod\limits_{j=1}^{k}(1-\nu_{i_j})^{q}\bigg)\bigg)^{\frac{1}{C_{n}^{k}}}
\bigg)^{\frac{1}{k}}\bigg)^{\frac{1}{q}}\bigg\rangle,
\end{align*}
which is exactly~\cite[Theorem~6]{LCW2020}.

(4) Let $q\geq 1$. If $f$ is the Archimedean t-conorm and t-norm based intuitionistic fuzzy
weighted averaging (ATS-IFWA) operator defined in \cite[Definition~6]{XXZ2012},
then, by \cite[Theorem~3]{XXZ2012} and formula~\eqref{eq-Trans}, we have that,
for $(\langle \mu_{1}, \nu_{1}\rangle, \cdots, \langle\mu_{n}, \nu_{n}\rangle)\in \tilde{\mathbb{Q}}^{n}$,
$$\mathscr{S}(\text{ATS-IFWA})(\langle \mu_{1}, \nu_{1}\rangle, \cdots, \langle\mu_{n}, \nu_{n}\rangle)
=\left\langle \sqrt[q]{\zeta^{-1}\left(\sum_{i=1}^{n}\omega_i\zeta((\mu_{i})^{q})\right)},
\sqrt[q]{\tau^{-1}\left(\sum_{i=1}^{n}\omega_i \tau((\nu_{i})^{q})\right)}\right\rangle,
$$
where $\tau$ is an additive generator of a strict t-norm and $\zeta(x)=\tau(1-x)$.
By Theorem~\ref{Equ-IFV-Q} and \cite[Properties~1--3]{XXZ2012},
$\mathscr{S}(\text{ATS-IFWA})$ is idempotent, increasing, and bounded.

(5) Let $q\geq 1$. If $f$ is Atanassov's intuitionistic fuzzy extended
Bonferroni mean (AIF-EBM) operator defined in~\cite[Definition~8]{DGM2017}, then,
by \cite[Theorem~1]{DGM2017} and formula~\eqref{eq-Trans}, we have that,
for $(\langle \mu_{1}, \nu_{1}\rangle, \ldots, \langle\mu_{n}, \nu_{n}\rangle)\in \tilde{\mathbb{Q}}^{n}$,
\begin{align*}
&\mathscr{S}(\text{AIF-EBM})(\langle \mu_{1}, \nu_{1}\rangle, \cdots, \langle\mu_{n}, \nu_{n}\rangle)\\
=&\Bigg\langle \sqrt[q]{1-\zeta^{-1}\left(\frac{1}{\lambda_1}\zeta(1-\zeta^{-1}(\zeta(A)+\zeta(B)))\right)},
\sqrt[q]{\zeta^{-1}\left(\frac{1}{\lambda_1}\zeta(1-\zeta^{-1}(\zeta(C)+\zeta(D)))\right)}\Bigg\rangle,
\end{align*}
where $\zeta$ is an additive generator of a strict t-conorm, $\lambda_1>0$,
$\lambda_2\geq 0$, $A=\zeta^{-1}\Big(\frac{n-|I^{\prime}|}{n}\zeta\Big(1-\zeta^{-1}
\Big(\frac{\lambda_1}{\lambda_1+\lambda_2}
\zeta\Big(1-\zeta^{-1}\Big(\frac{1}{n-|I^{\prime}|}\sum_{i\notin I^{\prime}}\zeta\Big\{1-\zeta^{-1}
\Big[\lambda_1 \zeta(1-(\mu_{i})^{q})+\zeta\Big(1-\zeta^{-1}\Big(\frac{1}{|I^{\prime}|}\sum_{j\in I^{\prime}}
\zeta(1-\zeta^{-1}(\lambda_2\zeta(1-(\mu_j)^{q})))\Big)\Big)\Big]\Big\}\Big)\Big)\Big)\Big)\Big)$,
$B=\zeta^{-1}\Big(\frac{1}{n}\sum_{i\in I^{\prime}}\zeta\Big(1-\zeta^{-1}(\lambda_1\zeta(1-(\mu_i)^{q}))\Big)\Big)$,
$C=\zeta^{-1}\Big(\frac{n-|I^{\prime}|}{n}\zeta\Big(1-\zeta^{-1}
\Big(\frac{\lambda_1}{\lambda_1+\lambda_2}
\zeta\Big(1-\zeta^{-1}\Big(\frac{1}{n-|I^{\prime}|}\sum_{i\notin I^{\prime}}\zeta\Big\{1-\zeta^{-1}
\Big[\lambda_1 \zeta((\nu_{i})^{q})+\zeta\Big(1-\zeta^{-1}\Big(\frac{1}{|I^{\prime}|}\sum_{j\in I^{\prime}}
\zeta(1-\zeta^{-1}(\lambda_2\zeta((\nu_j)^{q})))\Big)\Big)\Big]\Big\}\Big)\Big)\Big)\Big)\Big)$,
and $D=\zeta^{-1}\Big(\frac{1}{n}\sum_{i\in I^{\prime}}\zeta\Big(1-\zeta^{-1}(\lambda_1\zeta((\nu_i)^{q}))\Big)\Big)$.
By Theorem~\ref{Equ-IFV-Q} and \cite[P1--P3]{DGM2017},
$\mathscr{S}(\text{AIF-EBM})$ is idempotent, increasing, and bounded.
In particular, if $\zeta(t)=-\log(1-t)$, then $\mathscr{S}(\text{AIF-EBM})$ reduces to
\cite[Theorem~1]{LLL2018}.

(6) Similarly, all operators defined in \cite{LW2019,LW2020,DL2020,JMP2019,ZGGQHL2019}
can be obtained by formula~\eqref{eq-Trans} with appropriate operators on IFVs.

\section{An admissible similarity measure
with the order $\leq_{_{\text{XY}}}$ and its applications}\label{Sec-7}

Li and Cheng~\cite{LC2002} introduced the concept of the similarity measure for IFSs,
improved by Mitchell~\cite{Mit2003}, as follows. For more results on the similarity measure,
we refer to \cite{Sz2014}.

\begin{definition}[{\textrm{\protect\cite{Mit2003}}}]
\label{Def-Li-Cheng}
Let $X$ be a universe of discourse and $\mathbf{S}: \mathrm{IFS}(X)\times \mathrm{IFS}(X)
\rightarrow [0, 1]$ be a mapping. Mapping $\mathbf{S}(\_)$ is called an \textit{admissible similarity measure
with the order $\subset$} on $\mathrm{IFS}(X)$ if it satisfies the following conditions:
for any $I_1, I_2, I_3\in \mathrm{IFS}(X)$,
\begin{enumerate}[(1)]
  \item $0\leq \mathbf{S}(I_1, I_2)\leq 1$.
  \item $\mathbf{S}(I_1, I_2)=1$ if and only if $I_1=I_2$.
  \item $\mathbf{S}(I_1, I_2)=\mathbf{S}(I_2, I_1)$.
  \item If $I_1\subset I_2\subset I_3$, then $\mathbf{S}(I_1, I_3)\leq \mathbf{S}(I_1, I_2)$
  and $\mathbf{S}(I_1, I_3)\leq \mathbf{S}(I_2, I_3)$.
\end{enumerate}
\end{definition}

\begin{definition}
\label{Def-Wu-1}
Let $X$ be a universe of discourse and $I_1$, $I_2\in \mathrm{IFS}(X)$.
If $\langle \mu_{_{I_1}}(x), \nu_{_{I_1}}(x)\rangle \leq_{_{\text{XY}}}
\langle \mu_{_{I_2}}(x), \nu_{_{I_2}}(x)\rangle$ holds for all $x\in X$, then
we say that $I_{1}$ is \textit{smaller} than or equal to $I_2$ under the linear order $\leq_{_{\text{XY}}}$,
denoted by $I_{1} \leq_{_{\text{XY}}} I_2$.
\end{definition}

Based on Definition~\ref{Def-Wu-1},
we introduce the improved similarity measure for IFSs as follows:

\begin{definition}
\label{Def-Wu-Sim}
Let $X$ be a universe of discourse and $\mathbf{S}: \mathrm{IFS}(X)\times \mathrm{IFS}(X)
\rightarrow [0, 1]$ be a mapping. Mapping $\mathbf{S}(\_)$ is called an {\it admissible similarity measure
with the order $\leq_{_{\text{XY}}}$} on $\mathrm{IFS}(X)$ if it satisfies the
conditions (1)--(3) in Definition~\ref{Def-Li-Cheng}, and also the following one:
\begin{enumerate}
  \item[(4$^{\prime}$)] For any $I_1$, $I_2$, $I_3\in \mathrm{IFS}(X)$,
  if $I_1\leq_{_{\text{XY}}} I_2\leq_{_{\text{XY}}} I_3$, then $\mathbf{S}(I_1, I_3)\leq \mathbf{S}(I_1, I_2)$
  and $\mathbf{S}(I_1, I_3)\leq \mathbf{S}(I_2, I_3)$.
\end{enumerate}
\end{definition}

If $I_1\subset I_2$, then $I_{1} \leq_{_{\text{XY}}} I_2$. This implies that the function
$\mathbf{S}(\_)$, being the similarity measure under the order $\leq_{_{\text{XY}}}$ in Definition~\ref{Def-Wu-Sim},
is also the similarity measure in Definition~\ref{Def-Li-Cheng}. Namely, Definition~\ref{Def-Wu-Sim} is more substantial than
Definition~\ref{Def-Li-Cheng}. Definition~\ref{Def-Wu-Sim} includes the linear order $\leq_{_{\text{XY}}}$ on
$\tilde{\mathbb{I}}$ in consideration; however, Definition~\ref{Def-Li-Cheng} only considers the partial order
$\subset$ induced by $\cap$. Therefore, the similarity measure in Definition~\ref{Def-Wu-Sim} is more effective
than that of Definition~\ref{Def-Li-Cheng}. In the following, we construct such a similarity measure
$\mathbf{S}(\_)$ satisfying the conditions in Definition~\ref{Def-Wu-Sim}.

For $\alpha, \beta\in \tilde{\mathbb{I}}$, define
\begin{equation}
\label{rho-equ}
\varrho(\alpha, \beta)=
\begin{cases}
\frac{1}{3}(1+|s(\alpha)-s(\beta)|), & s(\alpha)\neq s(\beta), \\
\frac{1}{3}(|h(\alpha)-h(\beta)|), & s(\alpha)=s(\beta),
\end{cases}
\end{equation}
where $S(\alpha)$ and $h(\alpha)$ are the score degree and the accuracy degree of $\alpha$, respectively.


Now, we show some desirable properties of the function $\varrho$.

\begin{property}
\label{P-1}
(1) $\varrho(\alpha, \beta)\in [0, 1]$ and $\varrho(\alpha, \beta)=0$ if and only if $\alpha=\beta$.

(2) $\varrho(\alpha, \beta)=1$ if and only if ($\alpha=\langle 0, 1\rangle$ and
$\beta=\langle 1, 0\rangle$) or ($\alpha=\langle 1, 0\rangle$ and
$\beta=\langle 0, 1\rangle$).
\end{property}

\begin{property}
$\varrho(\alpha, \beta)=\varrho(\beta, \alpha)$.
\end{property}

\begin{property}
\label{P-3}
For any $\alpha, \beta, \gamma \in \tilde{\mathbb{I}}$,
$\varrho(\alpha, \beta)+\varrho(\beta, \gamma)\geq \varrho(\alpha, \gamma)$.
\end{property}

\textit{Proof of Property~\ref{P-3}.} Consider the following cases:

(1) If $s(\alpha)=s(\beta)=s(\gamma)$, then $\varrho(\alpha, \beta)+\varrho(\beta, \gamma)
=\frac{1}{3}(|h(\alpha)-h(\beta)|)+\frac{1}{3}(|h(\beta)-h(\gamma)|)
\geq \frac{1}{3}(|h(\alpha)-h(\gamma)|)= \varrho(\alpha, \gamma)$.

(2) If $s(\alpha)=s(\beta)\neq s(\gamma)$, then $\varrho(\alpha, \beta)+\varrho(\beta, \gamma)
=\frac{1}{3}(|h(\alpha)-h(\beta)|)+\frac{1}{3}(1+|s(\beta)-s(\gamma)|)\geq
\frac{1}{3}(1+|s(\beta)-s(\gamma)|)=\frac{1}{3}(1+|s(\alpha)-s(\gamma)|)=
\varrho(\alpha, \gamma)$.

(3) If $s(\alpha)\neq s(\beta)= s(\gamma)$, then, similarly to the proof of (2),
$\varrho(\alpha, \beta)+\varrho(\beta, \gamma)\geq \varrho(\alpha, \gamma)$.

(4) If $s(\alpha)=s(\gamma)\neq s(\beta)$, then $\varrho(\alpha, \beta)+\varrho(\beta, \gamma)
=\frac{1}{3}(1+|s(\alpha)-s(\beta)|)+\frac{1}{3}(1+|s(\beta)-s(\gamma)|)\geq
\frac{1}{3}\geq \frac{1}{3}(|h(\alpha)-h(\gamma)|)=\varrho(\alpha, \gamma)$.

(5) If $s(\alpha)\neq s(\beta)$, $s(\alpha)\neq s(\gamma)$, and
$s(\beta)\neq s(\gamma)$, then $\varrho(\alpha, \beta)+\varrho(\beta, \gamma)
=\frac{1}{3}(1+|s(\alpha)-s(\beta)|)+\frac{1}{3}(1+|s(\beta)-s(\gamma)|)
\geq \frac{1}{3}(2+|s(\alpha)-s(\gamma)|)>\varrho(\alpha, \gamma)$.

\begin{property}
\label{P-4}
For any $\alpha, \beta, \gamma \in \tilde{\mathbb{I}}$, if $\alpha\leq_{_{\text{XY}}}
\beta \leq_{_{\text{XY}}} \gamma$, then $\varrho(\alpha, \beta)\leq \varrho(\alpha, \gamma)$
and $\varrho(\beta, \gamma)\leq \varrho(\alpha, \gamma)$.
\end{property}

\textit{Proof of Property~\ref{P-4}.} Consider the following cases:

(1) If $s(\alpha)=s(\beta)=s(\gamma)$, then, by $\alpha\leq_{_{\text{XY}}}
\beta \leq_{_{\text{XY}}} \gamma$, we have that $h(\alpha)\leq h(\beta)\leq h(\gamma)$.
In this case, $\varrho(\alpha, \beta)=\frac{1}{3}(h(\beta)-h(\alpha))\leq
\frac{1}{3}(h(\gamma)-h(\alpha))=\varrho(\alpha, \gamma)$ and
$\varrho(\beta, \gamma)=\frac{1}{3}(h(\beta)-h(\beta))\leq \frac{1}{3}(h(\gamma)-h(\alpha))
=\varrho(\alpha, \gamma)$.

(2) If $s(\alpha)=s(\beta)\neq s(\gamma)$, then, by applying $\alpha\leq_{_{\text{XY}}}
\beta \leq_{_{\text{XY}}} \gamma$, we have that $h(\alpha)\leq h(\beta)$ and
$s(\alpha)=s(\beta)<s(\gamma)$. In this case, $\varrho(\alpha, \beta)=\frac{1}{3}(|h(\alpha)-h(\beta)|)
\leq \frac{1}{3}\leq \frac{1}{3}(1+|s(\alpha)-s(\gamma)|)=\varrho(\alpha, \gamma)$
and $\varrho(\beta, \gamma)=\frac{1}{3}(1+|s(\beta)-s(\gamma)|)
=\frac{1}{3}(1+|s(\alpha)-s(\gamma)|)=\varrho(\alpha, \gamma)$.

(3) If $s(\alpha)\neq s(\beta)= s(\gamma)$, then, by applying $\alpha\leq_{_{\text{XY}}}
\beta \leq_{_{\text{XY}}} \gamma$, we have that $h(\beta)\leq h(\gamma)$ and
$s(\alpha)<s(\beta)=s(\gamma)$. In this case, $\varrho(\alpha, \beta)=
\frac{1}{3}(1+|s(\alpha)-s(\beta)|) =\frac{1}{3}(1+|s(\alpha)-s(\gamma)|)
=\varrho(\alpha, \gamma)$ and $\varrho(\beta, \gamma)=
\frac{1}{3}(|h(\beta)-h(\gamma)|)\leq \frac{1}{3}\leq
\frac{1}{3}(1+|s(\alpha)-s(\gamma)|)
=\varrho(\alpha, \gamma)$.

(4) If $s(\alpha)=s(\gamma)\neq s(\beta)$, then, by applying $\alpha\leq_{_{\text{XY}}}
\beta \leq_{_{\text{XY}}} \gamma$, we have that $s(\alpha)<s(\beta)$ and $s(\beta)<s(\gamma)$.
This is impossible.

(5) If $s(\alpha)\neq s(\beta)$, $s(\alpha)\neq s(\gamma)$, and
$s(\beta)\neq s(\gamma)$, then, by applying $\alpha\leq_{_{\text{XY}}}
\beta \leq_{_{\text{XY}}} \gamma$, we have that $s(\alpha)<s(\beta)<s(\gamma)$.
In this case, $\varrho(\alpha, \beta)=\frac{1}{3}(1+|s(\alpha)-s(\beta)|)
<\frac{1}{3}(1+|s(\alpha)-s(\gamma)|)= \varrho(\alpha, \gamma)$
and $\varrho(\beta, \gamma)=\frac{1}{3}(1+|s(\beta)-s(\gamma)|)
<\frac{1}{3}(1+|s(\alpha)-s(\gamma)|)=\varrho(\alpha, \gamma)$.

Let $X=\{x_1, x_2, \ldots, x_n\}$ be a finite universe of discourse and
$\omega=(\omega_1, \omega_2, \ldots, \omega_n)^{\top}$ be the weight vector of $x_{j}$
($j=1, 2, \ldots, n$) such that $\omega_{j}\in (0, 1]$ and
$\sum_{j=1}^{n}\omega_j=1$. For $I_1, I_2\in \mathrm{IFS}(X)$,
define
\begin{equation}
\label{eq-dis-Wu}
d(I_1, I_2)=
\sum_{j=1}^{n}\omega_j\varrho(I_{1}(x_j), I_{2}(x_j)),
\end{equation}
and
\begin{equation}
\label{eq-Sim}\mathbf{S}(I_1, I_2)=1-\sum_{j=1}^{n}\omega_j\varrho(I_{1}(x_j), I_{2}(x_j)),
\end{equation}
where $I_{1}(x_j)=\langle \mu_{_{I_1}}(x_j), \nu_{_{I_1}}(x_j)\rangle$
and $I_{2}(x_j)=\langle \mu_{_{I_2}}(x_j), \nu_{_{I_2}}(x_j)\rangle$.

\begin{theorem}
\label{Thm-Sim-Wu}
(1) The function $d(\_)$ defined by formula~\eqref{eq-dis-Wu} is a metric
on $\text{IFSs}(X)$.

(2) The function $\mathbf{S}(\_)$ defined by formula~\eqref{eq-Sim} is a similarity measure
on $\text{IFSs}(X)$ under the order $\leq_{_{\text{XY}}}$.
\end{theorem}

\begin{proof}
It follows directly from Properties~\ref{P-1}--\ref{P-4}.
\end{proof}

\begin{remark}
(1) From Theorem~\ref{Thm-Sim-Wu}, it can be verified that $d(\_)$ is a new metric on IFSs,
which also partially answers Problem~\ref{Prob-1}.

(2) The topology $\mathscr{T}_1$ on the space $\tilde{\mathbb{I}}$
induced by the metric $\varrho$, given by formula~\eqref{rho-equ},
is weaker than the order topology $\mathscr{T}_2$ by $\leq_{_{\text{XY}}}$.
Notice from Corollary~\ref{Corollary-not-Metri} that $(\tilde{\mathbb{I}},
\mathscr{T}_2)$ is not metrizable, which means that $\mathscr{T}_1$
induced by the metric $\varrho$, given by formula~\eqref{rho-equ},
is strictly weaker than $\mathscr{T}_2$.
\end{remark}

\begin{example}[{\textrm{\protect\cite[Example~3.3.1]{XC2012}, \cite[Example~4.1]{WX2005}}}]
(A pattern recognition problem about the classification of building
materials) Consider four building materials: sealant, floor varnish, wall paint,
and polyvinyl chloride flooring represented by the IFSs $I_j$
($j=1, 2, 3, 4$) in the feature space $X=\{x_1, x_2, \ldots, x_{12}\}$.
Now, given another kind of unknown building material $I$ with data as listed
in Table~\ref{Tab-1}, we can use the similarity measure in Theorem~\ref{Thm-Sim-Wu}
to identify which type of $I_j$ the unknown material $I$ belongs.

\begin{table}[H]	
	\centering
	\caption{The data on building materials}
	\label{Tab-1}
     \scalebox{0.8}{
	\begin{tabular}{cccccc}
		\toprule
		 & $I_{1}$ &  $I_{2}$ & $I_{3}$ & $I_{4}$ & $I$\\
		\midrule
		$x_{1}$ & $\langle0.173, 0.524\rangle$ & $\langle0.510, 0.365\rangle$ &
           $\langle0.495, 0.387\rangle$ & $\langle1.000, 0.000\rangle$ & $\langle0.978, 0.003\rangle$\\
		$x_{2}$ & $\langle0.102, 0.818\rangle$ & $\langle0.627, 0.125\rangle$ &
           $\langle0.603, 0.298\rangle$ & $\langle1.000, 0.000\rangle$ & $\langle0.980, 0.012\rangle$\\
        $x_{3}$ & $\langle0.530, 0.326\rangle$ & $\langle1.000, 0.000\rangle$ &
           $\langle0.987, 0.006\rangle$ & $\langle0.857, 0.123\rangle$ & $\langle0.798, 0.132\rangle$\\
        $x_{4}$ & $\langle0.965, 0.008\rangle$ & $\langle0.125, 0.648\rangle$ &
           $\langle0.073, 0.849\rangle$ & $\langle0.734, 0.158\rangle$ & $\langle0.693, 0.213\rangle$\\
        $x_{5}$ & $\langle0.420, 0.351\rangle$ & $\langle0.026, 0.823\rangle$ &
           $\langle0.037, 0.923\rangle$ & $\langle0.021, 0.896\rangle$ & $\langle0.051, 0.876\rangle$\\
        $x_{6}$ & $\langle0.008, 0.956\rangle$ & $\langle0.732, 0.153\rangle$ &
           $\langle0.690, 0.268\rangle$ & $\langle0.076, 0.912\rangle$ & $\langle0.123, 0.756\rangle$\\
        $x_{7}$ & $\langle0.331, 0.512\rangle$ & $\langle0.556, 0.303\rangle$ &
           $\langle0.147, 0.812\rangle$ & $\langle0.152, 0.712\rangle$ & $\langle0.152, 0.721\rangle$\\
		$x_{8}$ & $\langle1.000, 0.000\rangle$ & $\langle0.650, 0.267\rangle$ &
           $\langle0.213, 0.653\rangle$ & $\langle0.113, 0.756\rangle$ & $\langle0.113, 0.732\rangle$\\
        $x_{9}$ & $\langle0.215, 0.625\rangle$ & $\langle1.000, 0.000\rangle$ &
           $\langle0.501, 0.284\rangle$ & $\langle0.489, 0.389\rangle$ & $\langle0.494, 0.368\rangle$\\
        $x_{10}$ & $\langle0.432, 0.534\rangle$ & $\langle0.145, 0.762\rangle$ &
           $\langle1.000, 0.000\rangle$ & $\langle1.000, 0.000\rangle$ & $\langle0.987, 0.000\rangle$\\
        $x_{11}$ & $\langle0.750, 0.126\rangle$ & $\langle0.047, 0.923\rangle$ &
           $\langle0.324, 0.483\rangle$ & $\langle0.386, 0.485\rangle$ & $\langle0.376, 0.423\rangle$\\
        $x_{12}$ & $\langle0.432, 0.432\rangle$ & $\langle0.760, 0.231\rangle$ &
           $\langle0.045, 0.912\rangle$ & $\langle0.028, 0.912\rangle$ & $\langle0.012, 0.897\rangle$\\
		\bottomrule
	\end{tabular}
      }
\end{table}

(1) If we consider that the weight vector $\omega$ of $x_j$ ($j=1, 2, \ldots, 12$) is:
\begin{align*}
\omega=(0.06, & 0.10, 0.08, 0.05, 0.10, 0.11, 0.09,\\
&0.06, 0.12, 0.10, 0.07, 0.06)^{\top},
\end{align*}
then, by formula~\eqref{eq-Sim}, we have that $\mathbf{S}(I_1, I)=0.3873$, $\mathbf{S}(I_2, I)=0.3828$,
$\mathbf{S}(I_3, I)=0.5437$, and $\mathbf{S}(I_4, I)=0.6491$. In this case,
$$
\mathbf{S}(I_4, I)>\mathbf{S}(I_3, I)>\mathbf{S}(I_1, I)>\mathbf{S}(I_2, I),
$$
which means that the unknown building material $I$ should approach $I_4$.
This result coincides with the one reported in~\cite{XC2012}.

(2) If the weights of $x_j$ ($j=1, 2, \ldots, 12$) are equal, i.e.,
$\omega_1=\omega_2=\cdots =\omega_{12}=\frac{1}{12}$, then,
by formula~\eqref{eq-Sim}, we have that $\mathbf{S}(I_1, I)=0.3793$, $\mathbf{S}(I_2, I)=0.3773$,
$\mathbf{S}(I_3, I)=0.5354$, and $\mathbf{S}(I_4, I)=0.6500$. In this case,
$$
\mathbf{S}(I_4, I)>\mathbf{S}(I_3, I)>\mathbf{S}(I_1, I)>\mathbf{S}(I_2, I),
$$
which means that the unknown building material $I$ should approach $I_4$.
This result coincides with the one reported in~\cite{WX2005}.
\end{example}

\section{Conclusion}\label{Sec-8}
In this paper, we have partially answered the three open problems proposed by
Atanassov~\cite{Ata1999,Ata2012} and systematically studies the topological
and algebraic structures of the spaces $(\tilde{\mathbb{I}}, \leq_{_{\textmd{XY}}})$
and $(\tilde{\mathbb{I}}, \leq_{_{\textmd{ZX}}})$.
We first showed that the two spaces $(\tilde{\mathbb{I}}, \leq_{_{\textmd{XY}}})$ and
$(\tilde{\mathbb{I}}, \leq_{_{\textmd{ZX}}})$ are isomorphic. We then introduced a new operator ``$\neg$" for IFVs
using the linear order $\leq_{_{\textmd{XY}}}$. We also demonstrated that this operator is a strong
negation on $(\tilde{\mathbb{I}}, \leq_{_{\textmd{XY}}})$. Moreover, we presented the following results:
(1) $(\tilde{\mathbb{I}}, \leq_{_{\textmd{XY}}})$ and $(\tilde{\mathbb{I}}, \leq_{_{\textmd{ZX}}})$
are complete lattices and Kleene algebras. (2) $(\tilde{\mathbb{I}}, \leq_{_{\textmd{XY}}})$ and
$(\tilde{\mathbb{I}}, \leq_{_{\textmd{ZX}}})$ are not separable and metrizable but compact and
connected topological spaces. These results partially answer Problems~\ref{Prob-1}--\ref{Prob-3}
from some new perspectives. Furthermore, we constructed an isomorphism between
$(\tilde{\mathbb{I}}, \leq_{_{\text{XY}}})$ and $(\tilde{\mathbb{Q}}, \leq_{_{\text{LW}}})$. Finally, we introduced
the concept of admissible similarity measures with particular orders for IFSs, extending the
previous definition of the similarity measure for IFSs given in~\cite{LC2002}. We finally constructed an
admissible similarity measure with the linear order $\leq_{_{\text{XY}}}$ and effectively
applied it to a pattern recognition problem about the classification of building materials.


\end{document}